\documentclass[twoside]{article}
\usepackage[accepted]{aistats2024_arxiv}

\usepackage[round]{natbib}

\usepackage{bm}

\newcommand{\mA}{\mathcal{A}}
\newcommand{\mB}{\mathcal{B}}

\newcommand{\mE}{\mathcal{E}}

\newcommand{\mG}{\mathcal{G}}

\newcommand{\mS}{\mathcal{S}}

\newcommand{\mV}{\mathcal{V}}

\newcommand{\mX}{\mathcal{X}}

\newcommand{\thmref}[1]{Theorem~\ref{#1}}

\newcommand{\tilO}{\tilde{O}}
\newcommand{\Prob}{\mathbb{P}}
\newcommand{\Ex}{\mathbb{E}}
\newcommand{\Regret}{\mathrm{Regret}}
\newcommand{\ONSW}{\mathrm{ONSW}}
\newcommand{\NSW}{\mathrm{NSW}}

\newcommand{\Real}{\mathbb{R}}

\newcommand{\vep}{\hat{v}}
\newcommand{\vucb}{\hat{v}^{\mathrm{UCB}}} 
\newcommand{\vucbclip}{\hat{v}^{\mathrm{RUCB}}}

\newcommand{\DA}{\mathrm{DA}} 
\newcommand{\EtC}{\mathrm{EtC}} 
\newcommand{\bc}{\color{blue}} 
\newcommand{\rc}{\color{red}}

\newcommand{\argmax}{\mathop{\rm arg~max}\limits}
\newcommand{\argmin}{\mathop{\rm arg~min}\limits}
\newcommand{\Ind}{\mathbb{I}} 
\newcommand{\eps}{\epsilon}
\newcommand{\CDA}{C^{\DA}}
\newcommand{\nn}{\nonumber\\}
\newcommand{\proj}{\mathop{\rm Proj}\limits}

\newcommand{\true}{\mathrm{true}}
\newcommand{\UCB}{\mathrm{UCB}}
 
\newcommand{\vrucb}{\hat{v}^{\mathrm{RUCB}}} 
\newcommand{\RUCB}{\mathrm{RUCB}}
\newcommand{\Termtwo}{X}
\usepackage[utf8]{inputenc} 
\usepackage[T1]{fontenc}    
\usepackage{hyperref}       
\usepackage{url}            
\usepackage{booktabs}       
\usepackage{amsfonts}       
\usepackage{nicefrac}       
\usepackage{microtype}      
\usepackage{xcolor}         
\usepackage{comment}

\usepackage{times}
\usepackage{soul}
\usepackage{amsmath}
\usepackage{amsthm}
\usepackage{algorithm}
\usepackage{algorithmic}
\usepackage{here}
\usepackage{graphicx}
\usepackage{CJKutf8}
\newtheorem{example}{Example}
\newtheorem{theorem}{Theorem}
\newtheorem{lemma}{Lemma}

\usepackage{subcaption}
\usepackage{multicol}

\begin{document}
\begin{CJK}{UTF8}{min}
%

%

\twocolumn[

\aistatstitle{Learning Fair Division from Bandit Feedback}

\aistatsauthor{ Hakuei Yamada \And Junpei Komiyama \And Kenshi Abe \And Atsushi Iwasaki }

\aistatsaddress{ University of\\ Electro-Communications \And New York University \And CyberAgent, Inc. \And University of\\ Electro-Communications } ]

\begin{abstract}
This work addresses learning online fair division under uncertainty, where a central planner sequentially allocates items without precise knowledge of agents' values or utilities. Departing from conventional online algorithm, the planner here relies on noisy, estimated values obtained after allocating items. We introduce wrapper algorithms utilizing \textit{dual averaging}, enabling gradual learning of both the type distribution of arriving items and agents' values through bandit feedback. This approach enables the algorithms to asymptotically achieve optimal Nash social welfare in linear Fisher markets with agents having additive utilities. We establish regret bounds in Nash social welfare and empirically validate the superior performance of our proposed algorithms across synthetic and empirical datasets.
\end{abstract}

\section{Introduction}

Ensuring an equitable distribution of limited resources among individuals, or agents, is a crucial task, as it involves reconciling competing interests among the agents. In today's society, there is a growing demand for efficient and fair resource allocation. 
Given this context, the fair division problem has attracted substantial interest from the researchers. This problem involves the process of assigning items to individuals with idiosyncratic preferences~\citep{brams:1996,moulin:2003}.
Many studies in fair division are treated as a static problem, assuming that all information about the items is known beforehand.
However, in reality, it is uncommon for all items to arrive beforehand. 
\textit{Online fair division} address uncertainty about what type of items can arrive dynamically, and the term "online" in this context means that items must be irrevocably assigned to an agent~\citep{aleksandrov:aaai:2020}. 

This paper introduces another layer of uncertainty regarding the values or utilities that agents place on the items they receive. Unlike traditional online algorithms, a central planner does not have exact knowledge of the agents’ utilities. Instead, the planner relies on noisy, estimated utilities obtained after the items have been allocated, as often assumed in \textit{multi-armed bandit} problems. 

We propose wrapper algorithms employing \textit{dual averaging} (DA)~\citep{xiao_dual,gao:pace:2021}, enabling the gradual learning of both the distribution of arriving item types and the values attributed to these items by the agents through bandit feedback. 
\cite{gao:pace:2021} applied DA to the Eisenbeg-Gale (EG) convex program~\citep{eisenberg:ams:1959,GEB:2010}.
If items are divisible, the solution aligns with \textit{competitive equilibrium with equal income} (CEEI)~\citep{budish:jpe:2011}.
CEEI ensures both envy-freeness and Pareto optimality and coincides with an assignment that maximizes \textit{Nash social welfare} (NSW)~\citep{nash:ecma:1950}. 
\cite{gao:pace:2021} derived a fair assignment, 
specifically an asymptotically envy-free and Pareto optimal online assignment, assuming types of arriving items are drawn independently and identically distributed (i.i.d.)\footnote{The i.i.d. assumption has recently been relaxed by \cite{liao2022nonstationary} to some extent.} and the values for each agent must be precisely revealed before the assignment in a round. However, as demonstrated in the following examples, the precise observation of values is not always possible, necessitating the handling of noisy feedback regarding the values for the agent who has received the item. 

\begin{example}{\rm (Crowdsourcing)}
Consider the problem of allocating tasks (items) to workers (agents), where each worker has their own area of expertise and the quality of the task may vary. In such a scenario, the importance of learning becomes evident, as the expertise of each worker is typically unknown before the task is assigned. If the aim is to maximize the quality of the task, there is a risk that only a small group of workers will monopolize all the tasks. However, by maximizing the NSW, task assignments can be distributed more evenly among the workers. 
\end{example}
\begin{example}{\rm (Food recommendation on online food delivery)}
Consider recommending a restaurant (agent) to a user (item). To do so, we need to estimate the user's taste preferences based on their past orders, such as whether they have a preference for specific types of cuisines. 
If the goal is solely to maximize social welfare, there is a risk that a few popular restaurants may receive a large portion of the orders, leading to delays or unfulfilled orders.
Instead, distributing the workload equitably among restaurants enhances the long-term efficiency~\citep{yuyan2022kdd}. 
\end{example}
\begin{example}{\rm (Humanitarian aid)}
Consider the problem of providing a limited number of items to disaster-stricken areas, where the supply arrives in real-time, and the prompt distribution of the items is crucial. The value of the items for each district can only be determined through actual feedback, which implies that the situation is an online problem. In this context, the value of the items needs to be learned and updated as the supplies are being distributed.
\end{example}

To address these challenges, we introduce two novel algorithms: DA-Explore-then-Commit (DA-EtC) and DA-Upper Confidence Bound (DA-UCB), which combine dual averaging with multi-armed bandit strategies. These algorithms gradually learn and explore the values of items through bandit feedback and the resulting allocation approaches the solution of the EG program as the number of items grows. Our algorithms handle situations where agents do not precisely observe the values of allocated items by incorporating estimated values into the DA process.
First, analyzing DA-EtC are intricate due to the need to create \textit{virtual} values fed into DA, ensuring compatibility with the i.i.d. property assumed by~\cite{xiao_dual}. We demonstrate that DA-EtC achieves a regret of $\tilde{O}(T^{2/3})$\footnote{We use $\tilde{O}$ to denote a Landau notation that disregards a polylogarithmic factor.} by meticulously dissecting the regret into the estimation error of values and the online learning error of DA, relative to the number of items~$T$. This indicates the rate at which it converges to the true CEEI solution as the number of items grows. 

Second, we introduce DA-UCB, which employs the upper confidence bound and is designed to have better empirical performance than DA-EtC.  Analyzing DA-UCB is even more complicated than DA-EtC due to the non-i.i.d. nature of the virtual values, referred to here as UCB values. To make the analysis tractable, we devise a variant of DA-UCB where multiple instances of DA are executed, treating each estimator matrix as constant. This preserves the structure of DA-UCB, enabling the virtual values to be drawn in an i.i.d. manner. Consequently, we establish a regret upper bound of $\tilde{O}(\sqrt{T})$, meeting the lower bound of $\tilde{\Omega}(\sqrt{T})$ applicable to any algorithm.

Let us explore the rest of the literature. Finding fair division closely relates to computing market equilibria or competitive equilibria, which has been extensively studied in algorithmic game theory, mainly for the prominent case of Fisher markets~\citep{vazirani:2007,codenotti:2007}. This problem is highly challenging in general. However, in a static setting, if agents have linear, additive preferences, the problem can be reduced to an optimization known as the EG 
convex program~\citep{eisenberg:ams:1959,GEB:2010}. 
As we already noted, if the items are divisible, the solution coincides with CEEI%
and an assignment that maximizes NSW. 
In contrast, if the items are indivisible, the compatibility no longer holds. However, the solutions of the NSW maximization is considered to have several plausible properties~\citep{Caragiannis:1:2019}.\footnote{First, an NSW solution is envy-free up to one item
, meaning that each agent (almost) prefers her allocation better than the others. Second, it is Pareto optimal, meaning that no one can be better off without sacrificing the utility of some other agents. Third, it also has pairwise maximin share guarantee, meaning that the division aligns with the intention of each agent.}

In online settings, even if items are divisible, envy-freeness is not compatible with Pareto optimality. Some researchers have taken a relaxed approach to envy-freeness, while others have found an approximate solution through stochastic approximation schemes: 
\cite{kash:jair:2014} relax envy-freeness for their online setting. 
\cite{bateni2022fair} utilize a stochastic approximation scheme and obtain an approximate solution for the EG convex program.
\cite{sinclair:sigmetric:2022} examine the trade-off between envy-freeness and Pareto optimality. The trade-off can be resolved if we allow the number of items to be large, or if we approximate indivisible items as divisible ones. Our work aims to achieve envy-freeness and Pareto optimality in the sense of asymptotically maximizing NSW as the number of items grows. 

\section{Problem setup}\label{sec:problem setup}
This section introduces the notation for our paper. Summarized notation is found in Appendix~\ref{sec_notation}. 
Let $N=\{1,2,...,n\}$ denote the set of agents, and $M=\{1,2,...,m\}$ denote the set of the types of items ($|N|=n, |M|=m$). The \textit{ex ante} (expected) value of each agent $i$ on an item of type $j$ is denoted as $v_{i,j} \in \Real$, which is unknown in advance to a policymaker or an algorithm.
At each round $t = \{1,2,...,T\}$, an (indivisible) item $j(t) \in M$ arrives and the policymaker allocates it to agent $i(t) \in N$. 
The type $j(t) \in M$ of item at $t$ is drawn from a certain probability distribution $\mathcal{S}=\{s_1,s_2,...,s_m\}$ in an i.i.d. manner, where $\sum_{j\in M} s_j=1$ and $s_j\geq 0$ for all $j\in M$.
When the agent $i(t)$ receives the item $j(t)$, the agent observes the \textit{ex post} (realized) utility of $u_i(t)=v_{i(t),j(t)}+\varepsilon_t$, instead of directly observing $v_{i(t),j(t)}$. The quantity $\varepsilon_t$ is a sub-Gaussian random variable with its radius $\sigma^2$. Namely, it satisfies
$
\Ex[ e^{\lambda \varepsilon_t} ] \le \exp\left(\lambda^2 \sigma^2/2\right),
$
which implies $\varepsilon_t$ to be mean zero.\footnote{The class of sub-Gaussian distributions includes many light-tail distributions. For example, it allows $u_i(t)$ to be a discrete value such as $\{0,1\}$ (Bernoulli distributions), or $\{1,2,3,4,5\}$ (categorical distributions).}
The cumulative utility of each agent $i$ in $T$ rounds, assuming the agents' values are additive, is defined as:
$U_i(T)=\sum_{t:i(t)=i} u_i(t)$.
The additive assumption is common in fair division~\citep{procaccia:ec:2014,Caragiannis:1:2019,bouveret:comsoc:2016}. 

This paper aims to asymptotically find a fair allocation in this online setting. 
In the literature, there are several ways to define fairness: \textit{Envy-freeness} guarantees that no agent has no envy toward another agent's allocation; \textit{Maximin share} guarantees that the minimum value among allocated agents is maximized. 
In general, both cannot be guaranteed in our online setting (See also the extensive survey~\citep{aleksandrov:aaai:2020}). 
We concentrate on asymptotically maximizing \textit{NSW}, defined as the geometric mean of the agents' obtained values, which is known as empirically balancing allocations between envy-freeness and \textit{Pareto efficiency}. 

For our noisy, online setting, we define the \textit{ex post} notion of NSW across time, which is equivalent to the weighted geometric mean of realized utilities: 
$\mathrm{NSW}(T)=
\prod_{i \in N} U_i(T)^{B_i}$.
The weights $B_i>0$ indicate the priority given to an agent, and they can be interpreted as her per-round budget rate in a Fisher market, e.g.,~\citep{gao:pace:2021}. 
Without loss of generality, we assume $\sum_i B_i = 1$.
We next consider a hindsight optimal allocation as a benchmark.
The law of large numbers implies that, when we allocate an item of type $j$ to agent $i$ many times the mean utility converges to $v_{i,j}$, and when the number of items is sufficiently large, we can approximate the items to be divisible. By using these facts, the hindsight optimal allocation is represented as the \textit{EG} convex program. 

Let $x_{i,j} \in [0,1]$ be the fraction of items $j$ allocated to agent $i$, consider the optimization per item:
\begin{align}\label{ineq_ONSW}
       \begin{aligned}
           & \text{maximize}_{\{x_{i,j}\}}&\ & 
           \prod_{i \in N} \left(\sum_{j \in M} s_jv_{i,j}x_{i,j}\right)^{B_i}
           \\& \text{subject to}&\ & \forall j \in M: \sum_{i \in N} x_{i,j} \leq 1, \\
           & &\ &\forall i \in N, \ \forall j \in M: x_{i,j}\ge 0.
     \end{aligned}
 \end{align}
We say an allocation \textit{optimal} if it maximizes the objective for the ex ante values and say the hindsight value of the optimal allocation \textit{Optimal NSW} (ONSW). 

We adopt a metric \textit{Regret} to measure the performance of an online allocation. The regret 
is the difference between the total hindsight utilities ONSW obtained from Eq.~\eqref{ineq_ONSW} and the ex post NSW with time horizon~$T$: 
 \begin{equation}
\begin{aligned}\label{ineq_def_regret}
\Regret(T) = T \cdot \ONSW - \NSW(T).
\end{aligned}
\end{equation}

It is not very difficult to see that the maximization of social welfare (SW) $\sum_i U_i(T)$ does not necessarily maximize the NSW. For example, if there is an agent $i$ with a very small value $v_{i,j}$ for all types $j$, then the social welfare is maximized by allocating no item to the agent, which results in zero NSW. Note also that, unlike SW, NSW is free from normalization; multiplying a constant on an $i$-th row of a matrix $\{v_{i,j}\}$ does not affect the optimal allocation.
Without loss of generality, the objective function of Eq.~\eqref{ineq_ONSW} can be replaced with the sum of weighted logarithmic utilities, 
i.e., $\sum_{i\in N}B_{i} \mathrm{log}\, \sum_{j \in M}v_{i,j}x_{i,j}$. 
We call the primal form $P_{EG}$ and

its dual form $D_{EG}$ given as 
\begin{equation*}
    \begin{aligned}
    & \text{minimize}_{\{\beta_i\}, \{p_j\}}&\ & \sum_{j \in M} s_j p_j-\sum_{i\in N}B_{i} \mathrm{log}\beta_i \\
    & \text{subject to}&\ & \forall i \in N, \ \forall j \in M:\ p_{j}\geq \beta_i v_{i,j}, \\
    & &\ & \forall j \in M:\ p_j \geq 0.
    \end{aligned}
 \end{equation*}
The value $p_j$ implies the \textit{price} of item $j$. 
This program has no duality gap and it belongs to a \textit{rational convex program}~\citep{vazirani:acm:2012} where all parameters are rational numbers and which always admits a rational solution. The program, while being nonlinear, can be solved by algorithms such as the ellipsoid method, e.g.,~\citep{vishnoi:2021}. 

However, solving EG once is insufficient for our aim. This is because (i) we do not know the values $\{v_{i,j}\}_{i,j}$ and need to update their estimates or empirical means for each round. (ii) Allocating items in a greedy manner based on the current estimates of $\{v_{i,j}\}_{i,j}$ results in a suboptimal NSW. In the following, we discuss ideas that address the issues above. 

\section{Dual-Averaging for Our Setting}
\label{sec_basic}

DA is an iterative method for solving a convex optimization problem\footnote{The version of DA we consider involves a regularization term and is sometimes referred to as regularized dual averaging.}, 
\cite{gao:pace:2021} have utilized DA for solving the dual problem $D_{EG}$. 
A variant of DA for our setting, described in Algorithm~\ref{DAbase}, calls a subroutine (Algorithm~\ref{DAiter}) to determine the winner at each round. 
%
Algorithm~\ref{DAiter} uses a multiplier $\beta_i \in \mathbb{R}^+$ to balance the allocation of items among agents and prevent any one agent from monopolizing the allocation. 
As agent~$i$ receives more items, $\beta_i$ decreases. Each iteration in DA corresponds to a round, which is an analog of the arrival of an item, in the online fair division, where an item with value 
$\hat{v}^{\DA}_{i,j(t)}(t)$ is allocated according to a (virtual) first-price auction, where the bids of agents are weighted by the multiplier $\beta_i v_i^{DA}$. 
Here, the parameters $l, h, \delta_0 > 0$ restrict the range of $\beta_i$ for the sake of stability.
At round $t$, the winner $i(t)$ of the auction receives the item and obtains the utility $u^{\DA}_i(t)$. If the true values $\{v_{i,j(t)}\}$ are available, they are used in $\{\hat{v}^{\DA}_{i,j(t)}(t)\}$ of the subroutine. This is the case of the PACE algorithm~\citep{gao:pace:2021}, which has shown to have $\tilde{O}(1/\sqrt{T})$ regret. 

However, Algorithm~\ref{DAbase} requires the true value $v_{i,j(t)}$ for agent~$i$ receiving type $j$ item at round $t$, which is not available in our setting.
Instead, we consider plugging-in the estimated value to the subroutine (Algorithm~\ref{DAiter}) as described in the following section. 

\begin{figure}[tb]
\begin{algorithm}[H]
\caption{DA with true values $\{v_{i,j}\}$}
\label{DAbase}
\begin{algorithmic}[1]
\STATE Initialize utility $\bar{u}^{\DA}_i = 0$ for each $i$.
\FOR {$t = 1$ to $T$}
\STATE Observe the type $j(t)$ of an arriving item, drawn from~$\mathcal{S}$.
\STATE Observe the value $v_{i,j(t)}$ of the item for each $i$. 
\STATE Update the mean utilities $\{\bar{u}^{\DA}_i(t+1)\}_i = \text{DA-Iter}(t, \{v_{i,j(t)}\}_i, \{\bar{u}^{\DA}_i(t)\}_i)$
\ENDFOR
\end{algorithmic} 
\end{algorithm}%
\begin{algorithm}[H]
\caption{DA-Iter}
\label{DAiter}
\begin{algorithmic}[1]
\REQUIRE $t$, $\{\hat{v}^{\DA}_{i,j(t)}(t)\}_i$, $\{\bar{u}^{\DA}_i(t)\}_i$
\STATE Define multiplier $\beta_i$ as
\begin{align*}
\centering
\beta_i=\mathrm{Proj}_{[B_i/(h(1+\delta_0)),(1+\delta_0)/l]}(B_i/\bar{u}^{\DA}_i(t))\ \ \ (\delta_0>0),
\end{align*}
where $\mathrm{Proj}_{[a, b]}(x) = \max(a, \min(b, x))$.
\STATE Agents bid $\beta_i \hat{v}^{\DA}_{i,j(t)}(t)$.
\STATE Winner is determined: $i(t)={\argmax\limits_{i}\beta_i \hat{v}^{\DA}_{i,j(t)}(t)}$, where ties are broken arbitrarily.
\STATE Winner $i(t)$ pays $\beta_{i(t)} \hat{v}^{\DA}_{i,j(t)}(t)$.
\STATE Each agent receives a utility: $u^{\DA}_i(t) = \mathbb{I}\{i=i(t)\} \hat{v}^{\DA}_{i,j(t)}(t)$.
\STATE Update mean utility for each agent as
\begin{equation*}
\begin{aligned}
\bar{u}^{\DA}_i(t+1) =
\frac{t-1}{t} \bar{u}^{\DA}_i(t) +\frac{1}{t} u^{\DA}_i(t)
\end{aligned}
\end{equation*}
\RETURN $\{\bar{u}^{\DA}_i\}_i$ (and $i(t)$ for DA-EtC, DA-UCB)
\end{algorithmic} 
\end{algorithm}
\end{figure}%

\section{Proposed Method}

\subsection{DA-EtC}
\label{DA-EtCsec}
Algorithm~\ref{DA-EtC} describes the procedure of DA-EtC.  
It begins with a uniform exploration phase, where each agent tries every item type equally for the first $T_0$ rounds. At the end of this phase, the algorithm generates an estimator $\hat{v}_{i,j}^{\EtC} = \hat{v}_{i,j}(T_0+1)$ of the expected value $v_{i,j}$ of each item type $j$ for each agent $i$ as in line 7. Here, $N_{i,j}(t)$ is the number of times agent $i$ received an item of type $j$ up to round $t-1$: $\sum_{\tau<t} \Ind[i(\tau)=i\wedge j(\tau)=j]$ where $\Ind[\cdot]$ is an indicator function. After the exploration phase, the algorithm runs a DA algorithm for the remaining $T - T_0$ rounds using the estimator $\{\hat{v}_{i,j}^{\EtC}\}$. In each of the rounds, the algorithm selects the winner $i(t)$ and updates the mean utilities $\bar{u}_i^{\DA}(t+1)$.

To implement DA-EtC, a certain level of attention is required. 
Specifically, DA-EtC uses $\hat{v}_{i,j}^{\EtC}$ 
for $\hat{v}^{\DA}_{i,j}(t)$ in 
DA-Iter and calculate $\bar{u}_i^{\DA}(t)$. 
Note that the estimator $\hat{v}_{i,j}^{\EtC}$ is fixed after round $T_0$, primarily due to theoretical reasons to derive 
Lemma~\ref{lem_dabound}.

\begin{figure}[tb]
\begin{algorithm}[H]
\caption{DA-EtC}
\label{DA-EtC}
\begin{algorithmic}[1]
\FOR {$t = 1$ to $T_0$}
    \STATE Observe the type $j(t)$ of an arriving item, drawn from~$\mathcal{S}$.
    \STATE Give item $j(t)$ to the agent who is chosen uniformly at random.  
    \STATE The agent $i(t)$ receives a utility $u_{i(t)}(t)$.
    \STATE Update the cumulative utility: $U_{i}(t+1) = U_{i}(t) + \Ind[i=i(t)] u_{i(t)}(t)$ for each $i$.
\ENDFOR
\STATE Fix the estimator:
\[
\hat{v}_{i,j}^{\EtC}  = \frac{\sum_{t:j(t) = j} u_{i(t)}(t)\Ind[i=i(t)]}{N_{i,j}(T_0+1)}
\]
for each $i,j$.
\STATE Initialize the DA's utility $\bar{u}^{\DA}_i(1)=0$ for each $i$.
\FOR {$t = T_0 + 1$ to $T$}
\STATE The type of item $t$ is determined: $j(t) \sim \mS$.
\STATE $t' = t - T_0$
\STATE Run DA-Iter to allocate an item: $\{\bar{u}^{\DA}_i(t'+1)\}_i, i(t) = \text{DA-Iter}(t', \{\hat{v}_{i,j(t)}^{\EtC}\}_i, \{\bar{u}^{\DA}_i(t')\}_i)$.
\STATE The agent $i(t)$ receives a utility $u_{i(t)}(t)$.
\STATE Update the cumulative utility: $U_{i}(t+1) = U_{i}(t) +  u_{i(t)}(t)\Ind[i=i(t)]$ for each $i$.
\ENDFOR
\end{algorithmic} 
\end{algorithm}
\end{figure}

\subsection{DA-UCB}
\label{DA-UCBsec}
Despite being able to control the exploration duration $T_0$ for minimizing regret, DA-EtC has limited adaptivity. 
The largest limitation stems from the uniform exploration,
which may result in the unnecessary exploration of items with a significantly low value. 
To mitigate this, we develop DA-UCB in Algorithm~\ref{DA-UCB}, which employs an upper confidence bound that holds with high probability and 
solves the exploration and exploitation tradeoff more adaptively than DA-EtC. 
At each round $t$, DA-UCB calculates the UCB value $\vucb_{i,j(t)}(t)$ for each agent $i$. This value is then provided to DA-Iter to determine the winner~$i(t)$. 
Note that the utility with respect to DA, denoted as $u^{\DA}_i(t)$, differs from the actual utility $u_i(t)$ for both DA-EtC and DA-UCB. This is because the estimated value is supplied to DA.

\begin{figure}[tb]
 \begin{algorithm}[H]
 \caption{DA-UCB}
     \label{DA-UCB}
 \begin{algorithmic}[1]
\STATE Initialize the DA's utility $\bar{u}^{\DA}_i(1)=0$ for each $i$ and $\vucb_{i,j}(1)=1$ for all $i,j$.
  \FOR {$t = 1$ to $T$}
    \STATE Observe the type $j(t)$ of an arriving item, drawn from~$\mathcal{S}$.
  \STATE Calculate the UCB value 
  \begin{equation*}
\begin{aligned}
\vucb_{i,j(t)}(t)=\min\left(1,\vep_{i,j(t)}(t)+\sqrt{\log t / (2N_{i, j(t)}(t))}\right)\\ \text{ where } \min(1, +\infty) = 1.
\end{aligned}
\end{equation*}
\STATE Run DA-Iter to allocate an item: $\{\bar{u}^{\DA}_i(t+1)\}_i, i(t) = \text{DA-Iter}(t, \{\vucb_{i,j(t)}(t)\}_i, \{\bar{u}^{\DA}_i(t)\}_i)$.
\STATE The agent $i(t)$ receives a utility $u_{i(t)}(t)$.
\STATE Update the number of draws $N_{i(t),j(t)}(t+1) = N_{i(t),j(t)}(t) + 1$.
\STATE Update the cumulative utility: $U_{i}(t+1) = U_{i}(t) +  u_{i(t)}(t)\Ind[i=i(t)]$ for each $i$.
\STATE Update the estimator:
\\\[
\hat{v}_{i(t),j(t)}(t+1) = \frac{\sum_{\tau \le t: j(\tau) = j} u_{i(t)}(\tau)\Ind[i=i(t)]}{N_{i,j}(t+1)}
\]
  \ENDFOR
 \end{algorithmic} 
 \end{algorithm}
 \end{figure}

\section{Analysis}

In this section, we analyze the regret bounds of the algorithms. We begin by establishing a lower bound on the regret, which represents the best performance that any algorithm can achieve. Following that, we provide regret upper bounds on DA-EtC and a variant of DA-UCB. This is because directly analyzing DA-UCB poses significant challenges. 

\subsection{Regret lower bound}

\begin{theorem}{\rm (Regret lower bound)}\label{thm_reglower}
There exists a model where
the expected regret of any algorithm is lower-bounded as
$
\Ex[ 
\Regret(T)
]
=
\Omega( \sqrt{mT} ). 
$
\end{theorem}
\begin{proof}[Proof Sketch of Theorem \ref{thm_reglower}]
We introduce the \textit{base} model where $v_{i,j} = 1/2$ for all $i,j$. To maximize NSW, the algorithm must allocate (approximately) $T/n$ of items for each agent. 
Let $j_i$ indicate the type of item that agent $i$ receives the least number of times during $T$.
There exists a set of $n$ items $\{j_i\}_{i\in N}$ such that, $\sum_i N_{i, j_i} \le T/m$. 
We consider an alternative model such that $v_{i, j_i}$ is larger than $1/2$ by $\sqrt{m/T}$. To have a low regret in the alternative model, the algorithm must choose arms $\{j_i\}_{i\in N}$ frequently. However, information-theoretic results imply that the algorithm cannot differentiate the base model and the alternative model, and thus suffers a large regret in the alternative model~\citep{kaufmann16}.
\end{proof} 

\subsection{Convergence on Dual Averaging}

We state the convergence results of DA. The following lemma is an extension of Theorem 4 in~\citep{gao:pace:2021}. To be precise, Theorem 4 in~\citep{gao:pace:2021} requires the values to be normalized as $\sum_j s_j v_{i,j} = 1$ for each agent $i$, and thus it is not directly applicable to the setting where $v_{i,j}$ is unknown. Lemma~\ref{lem_dabound} here generalized their results by introducing $l,h$ so that normalization is no longer required.

\begin{lemma}\label{lem_dabound} 
Assume that: (a) we run DA for $T$ rounds; (b) 
let $\{v^{\DA}_{i,j}\}$ be $\hat{v}^{\DA}_{i,j(t)}(t)$ for all t, and they are constant;
Let $\bar{u}^{\DA}(T) \in \Real^n$ be the mean utility vector of agents at the end of round $T$ and $u^{*,\DA} \in \Real^n$ be the solution of the corresponding EG program with $\{v^{\DA}_{i,j}\}$.
Then, the following inequality holds for an arbitrary $\delta_0>0$:
\begin{align*}
\mathbb{E}\left[\|\bar{u}^{\DA}(T) - u^{\ast, \DA}\|^2\right] &\leq \CDA \frac{6 + \log T}{T}.
\end{align*}
where $\|\cdot\|$ is a l2 norm vector and 
\[
\CDA = \frac{\|v^{\DA}\|_{\infty}^2(1+\delta_0)^6}{l^4\left(\min_{i\in N}B_i\right)^4}\left(\frac{h^3\|v^{\DA}\|_{\infty}^2}{l}\left(\frac{1}{\delta_0}\right)^2 + h^4\right)
\] and $\|v^{\DA}\|_{\infty}=\max_{i\in N}\|v_i^{\DA}\|_{\infty}$.
\end{lemma}

\subsection{Regret upper bound of DA-EtC}

\begin{theorem}{\rm (Regret upper bound of DA-EtC)}\label{thm_etc}
Assume $B_i = 1/n$ and $s_j = 1/m$ for all $i,j$.
Assume that 
$2l \le v_{i,j} \le h/2$ for all $i,j$.
Then, for $T_0$ that is sufficiently large such that
$
\sqrt{\frac{8 \max(1, \sigma^2) nm\log(nmT)}{T_0}} \le \frac{\min(1, 2l)}{2}
$
holds, the expected regret of DA-EtC is bounded as follows:
\[
\mathbb{E}[\Regret(T)] = \tilde{O}\left( T_0 + \sqrt{\frac{nm}{T_0}} T + n \CDA\sqrt{T}\right) .
\]
The $\tilde{O}(n \CDA\sqrt{T})$ term is non-leading with respect to $T$.
By focusing on the leading term in $T$ and setting $T_0 = T^{2/3}(nm)^{1/3}$, we obtain
$
\mathbb{E}[\Regret(T)] = \tilO(T^{2/3}(nm)^{1/3}).
$
\end{theorem}%
\begin{proof}[Proof Sketch of \thmref{thm_etc}]
We decompose the regret into two components.
Let $u^{*,\true} = (u_1^{*,\true},u_2^{*,\true},u_3^{*,\true},\dots,u_n^{*,\true})$
be the solution of EG with true values, i.e., Eq.~\eqref{ineq_ONSW}.
Let 
\[
\bar{u}_i^{\EtC} = \frac{\sum_{t > T_0}^T u_i(t) \Ind[i(t)=i]}{T-T_0}
\]
be the mean utility during the exploitation rounds. 
$u_i^{*,\DA}$ is the solution of EG where value matrix $\{\hat{v}_{i,j}^{\EtC}\}_{i,j}$
Note that $\bar{u}_i^{\DA}$ is the mean utility in view of DA-Iter, which corresponds to $\bar{u}_i^{\EtC}$ but $u_i(t)$ is replaced by $\hat{v}_{i,j(t)}^{\EtC}$. 
Then, we demonstrate a novel decomposition for the regret-per-round that boils it down into the following terms: 
\begin{align}
\lefteqn{
\frac{\Regret(T)}{T} 
 = 
\prod_i (u^{*,\true}_i)^{B_i} 
- 
\prod_i (\bar{u}_i^{\EtC})^{B_i} + O(T_0)
}\nn 
&=
(\prod_i (u^{*,\true}_i)^{B_i} - \prod_i (u_i^{*,\DA})^{B_i}) \nn
&\ \ \ \ + (\prod_i (u_i^{*,\DA})^{B_i} - \prod_i (\bar{u}_i^{\DA})^{B_i}) 
\nn
&
\ \ \ \ + (\prod_i (\bar{u}_i^{\DA})^{B_i}) - \prod_i (\bar{u}_i^{\EtC})^{B_i}) + O(T_0).
\label{ineq_etcbound_twoterms}
\end{align}
Intuitively, 
The first term of Eq.~\eqref{ineq_etcbound_twoterms} 
is due to the estimation error of $\{v_{i,j}\}_{i,j}$, which depends on the quality of estimator built at the end of the exploration phase. The second term of Eq.~\eqref{ineq_etcbound_twoterms} is the error due to the loss of DA, which we bound via Eq.~\eqref{lem_dabound} and by converting the error on the l2-norm into the error on the error of NSW. 
The third term is due to the estimation error $\{v_{i,j}\}_{i,j}$ as well as the random realization of the samples in the commitment phase.
\end{proof} 

\subsection{Regret upper bound of a variant of DA-UCB}

This section analyzes a slightly modified version of DA-UCB, which we call Repeated Dual Averaging UCB (RDA-UCB, Algorithm \ref{RDA-UCB}) to demonstrate that DA-UCB may be able to achieve the $\tilde{O}(\sqrt{T})$ regret. In general, the regret analysis depends on Lemma~\ref{lem_dabound}, which requires that inputs for DA are drawn in an i.i.d. manner. DA-EtC inevitably satisfies the requirement because it fixes a value matrix, whose elements $\hat{v}_{i,j}^{EtC}$ are the inputs for DA, at the end of the exploration period. In contrast, DA-UCB by no means satisfies the requirement, since it updates a UCB value $\hat{v}_{i,j}^{UCB}$ at each round. 
To overcome this, RDA-UCB executes multiple instances of DA and each of the estimator matrices becomes regarded as constant, so that it comes to satisfy the i.i.d. assumption. Note that, in practice, RDA-UCB is clearly outperformed by DA-UCB due to the complexity and thus we perform only DA-UCB in the simulation section. 

Let the projected UCB value when as 
\begin{equation}
\vucbclip_{i,j}(t) = \mathrm{Proj}_{[l, h]}\left(
\vep_{i,j}(t) + \sqrt{\frac{\log (T^2)}{2N_{i,j}(t)}}\right),
\end{equation}
where 
$\vep_{i,j}(t)$ is the empirical estimate of $v_{i,j}$ with $N_{i,j}(t)$ samples.
Here, we define 
$\vucb_{i,j}(1) = h$ for all $i,j$.
The RDA-UCB algorithm proceeds as follows. It begins by fixing a value matrix (Line \ref{line_ucbmat}) based on the current UCB value and then executes an instance of DA (Line \ref{line_rda}). This instance continues until an entry of the count matrix $N_{i,j}(t)$ reaches a power of $2$ (Line \ref{line_powertwo}), at which point the current instance of DA is terminated. Following the termination, a new value matrix is created (Line \ref{line_ucbmat}), and another instance of DA is run (Line \ref{line_rda}). This process is repeated until the number of rounds reaches $T$. The following theorem states the optimality of the regret of RDA-UCB.

\begin{figure}[tb]
 \begin{algorithm}[H]
 \caption{RDA-UCB}
     \label{RDA-UCB}
 \begin{algorithmic}[1]
  \STATE $s = 1$.
  \STATE Initialize $\vucbclip_{i,j}(1)=h$ for all $i,j$.
  \WHILE{$s \le T$}
    \STATE Fix the RUCB value matrix 
    $\vucbclip_{i,j}(s)$ This RUCB value remains the same during the inner loop. \label{line_ucbmat}
    \STATE Initialize the DA's utility $\bar{u}^{\DA}_i(1)=0$ for each $i$.
    \WHILE{true}\label{line_rda}
    \STATE $t = s$.
    \STATE Observe the type $j(t)$ of an arriving item, drawn from~$\mathcal{S}$.
    \STATE Run DA-Iter to allocate an item: $\{\bar{u}^{\DA}_i(t+1)\}_i, i(t) = \text{DA-Iter}(t-s+1, \{\vucbclip_{i,j(t)}(t)\}_i, \{\bar{u}^{\DA}_i(t)\}_i)$.
    \STATE The agent $i(t)$ receives a utility $u_{i(t)}(t)$.
    \STATE Update the number of draws $N_{i(t),j(t)}(t+1) = N_{i(t),j(t)}(t) + 1$.
    \STATE Update the cumulative utility: $U_{i}(t+1) = U_{i}(t) +  u_{i(t)}(t)\Ind[i=i(t)]$ for each $i$.
    \STATE Update the estimator $\hat{v}_{i(t),j(t)}(t) = \frac{U_{i}(t+1)}{N_{i,j}(t+1)}.$
    \IF{$N_{i(t),j(t)}(t) \in \{2,2^2,2^3,\dots,2^{\lfloor \log_2 T\rfloor}\}$} \label{line_powertwo}
      \STATE $s = t+1$ and break the inner loop.
    \ELSIF{$t \ge T$}
      \STATE Terminate the algorithm.    
    \ENDIF
    \STATE $t = t + 1$.
    \ENDWHILE
    \STATE $\vucbclip_{i,j}(t) = \mathrm{Proj}_{[l, h]}\left(
    \vep_{i,j}(t)+\sqrt{\frac{\log (T^2)}{2N_{i,j}(t)}}\right).$
  \ENDWHILE
 \end{algorithmic} 
 \end{algorithm}
 \end{figure}

\begin{theorem}{\rm (Regret bound of RDA-UCB)}\label{thm_ucb_reset}
Assume that $l,h$ satisfies that $l\leq v_{i,j} \leq h$ for all $i\in N, j \in M$.
Then, the following inequality holds:
\begin{align}
\Ex[\Regret(T)]
\le
\tilde{O}\left(
\mathrm{poly}(n, m)\sqrt{T}
\right),
\end{align}
where $\mathrm{poly}(n, m)$ is a polynomial of $n,m$ that is independent of $T$.
\end{theorem}

\subsection{Discussion on the rate of regret}

In summary, our theoretical contributions are twofold: 
i) DA converges in our setting where the values of agent-type pairs are unobservable beforehand and 
ii) DA-EtC has the upper and lower bound being dependent on the numbers of agents and item types and the time horizon. 
In what follows, we will examine and discuss what these results imply. 
Theorem~\ref{thm_etc} provides DA-EtC has the regret upper bound of $\tilde{O}((nm)^{1/3} T^{2/3})$. The exponent on $m$ here is smaller than that of lower bound of $\tilde{\Omega}((mT)^{1/2})$ in Theorem~\ref{thm_reglower}. This does not contradict because Theorem~\ref{thm_etc} assumes the exploration per parameter $T_0/(nm) = \Omega(1)$, or equivalently $T > nm$. Thus, for any $T, n, m$ such that $T > m$, the upper bound is no smaller than the lower bound.

We consider giving an upper bound for DA-UCB is challenging mainly because the i.i.d. assumption is crucial for deriving the performance bound in Lemma~\ref{lem_dabound}. 
However, if we plug-in the UCB values at each round $t$ to DA-Iter, it no longer satisfy the assumption. Started from \cite{xiao_dual}, most of existing results on DA require the data to be i.i.d. A few notable exceptions are as follows:  \cite{agrawal_depdata} used in DA for mixing process, but the application of their results to a regularized version of DA such as ours is highly non-trivial. Note also that the classes of non-stationarity that \cite{liao2022nonstationary} consider are not directly applicable to our settings.
It might be possible to apply the analyses by \cite{xiao_dual,liao2022nonstationary} to our case. However, the online regret term\footnote{Namely, $R_t(w)$ in Section 4.1 of \citep{liao2022nonstationary}. Note that this is different from the regret in our paper.} is very challenging to bound in our case.
To avoid the nonstationarity, we invented a theoretical algorithm that we call RDA-UCB. Although this algorithm employs multiple instance of DA, the number of instances is $O(nm \log(T)) = \tilde{O}(1)$ as a function of $T$, and thus it does not matter to the leading term of $\sqrt{T}$. As a result, it achieves the optimal rate $\tilde{O}(\sqrt{T})$ with respect to the number of rounds $T$. 
We consider the dependence of RDA-UCB to $n, m$ (available in the appendix) to be suboptimal, and a more practical algorithm is desired.

Let us finally compare our results with the rate of regret in the partial monitoring problems~\citep{bartok2014,komiyama2017}.
In particular, our results is $\tilde{\Theta}(T^{1/2})$ with respect to time horizon~$T$. 
A partial monitoring problem is said to be \textit{easy} if the chosen action itself defines how many items are assigned, e.g., multi-armed bandit problems~\citep{LAI1985}, which corresponds to $\tilde{\Theta}(T^{1/2})$ regret. Otherwise, the problem is hard, which corresponds to $\tilde{\Theta}(T^{2/3})$ regret. This categorization suggests that numerous problems involving parameters that are not directly observable from the optimal decision, such as dynamic pricing, inevitably have a regret complexity of $\tilde{\Theta}(T^{2/3})$.
Our results indicate that the problem we consider belongs to the former (i.e., $\tilde{O}(\sqrt{T})$) class of partial monitoring. 
This is interesting because we have latent parameters (coefficients $(\beta_i)_{i\in N}$) that depend on all agents and thus are not solely determined by the values of particular user $i$. 
This instills hope that a large class of bandit optimizations that entails optimizing latent parameters can achieve $\tilde{O}(\sqrt{T})$ regret. 
We also note that the comparison with bandit convex optimization and our problem is in Section \ref{sec_bandit_convex} in the appendix.

\section{Simulations}\label{sec:simulation}

\begin{figure*}[tb]
\begin{minipage}[t]{0.33\linewidth}
    \centering
    \includegraphics[keepaspectratio,width=\linewidth]{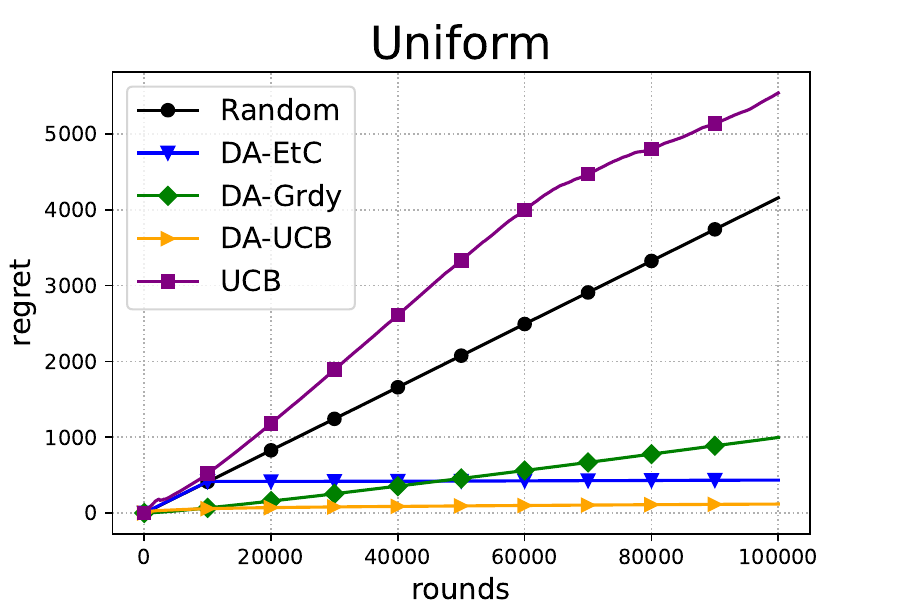}
    \subcaption{Uniform (n=10 and m=10)}\label{fig:all1}
\end{minipage}
\begin{minipage}[t]{0.33\linewidth}
    \centering
    \includegraphics[keepaspectratio,width=\linewidth]{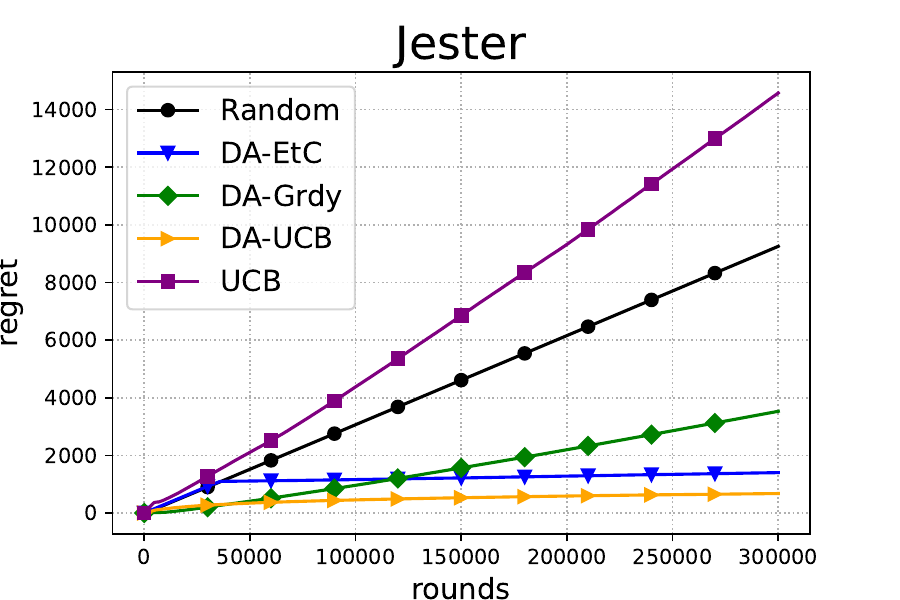}
    \subcaption{Jester (n=10 and m=50)}\label{fig:all2}
\end{minipage}
\begin{minipage}[t]{0.33\linewidth}
    \centering
    \includegraphics[keepaspectratio,width=\linewidth]{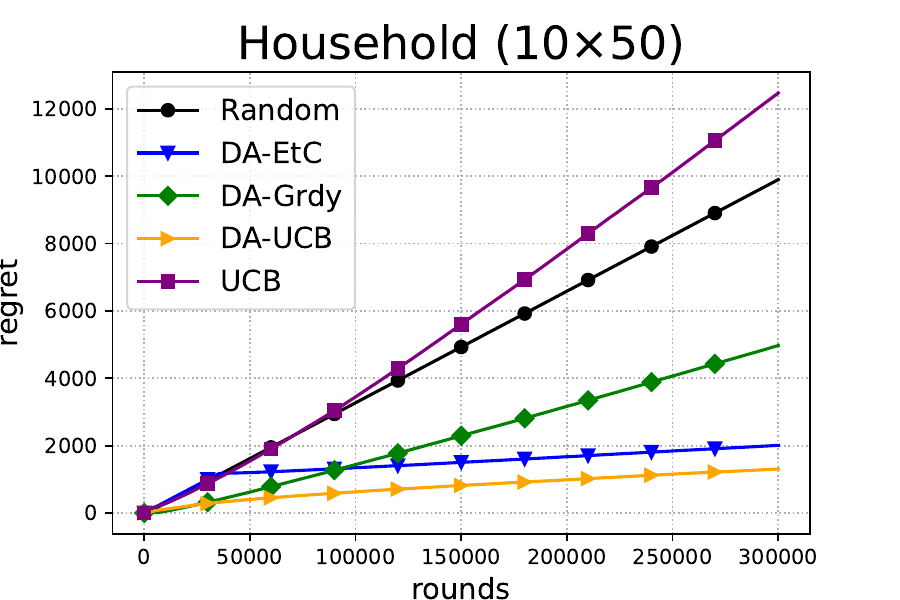}
    \subcaption{Household (n=10 and m=50)}\label{fig:all3}
\end{minipage}
\caption{Regret of all algorithms. 
}\label{fig:all-wide}
\end{figure*}
\begin{figure*}[tb]
\begin{minipage}[t]{0.33\linewidth}
    \centering
    \includegraphics[keepaspectratio,width=\linewidth]{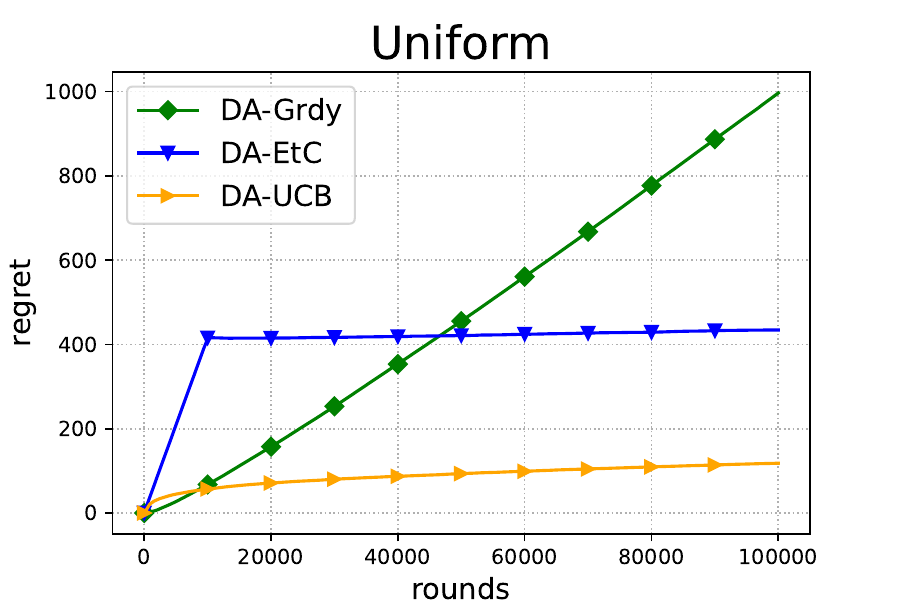}
    \subcaption{Uniform (n=10 and m=10)}\label{fig:EtC_UCB1}
\end{minipage}
\begin{minipage}[t]{0.33\linewidth}
    \centering
    \includegraphics[keepaspectratio,width=\linewidth]{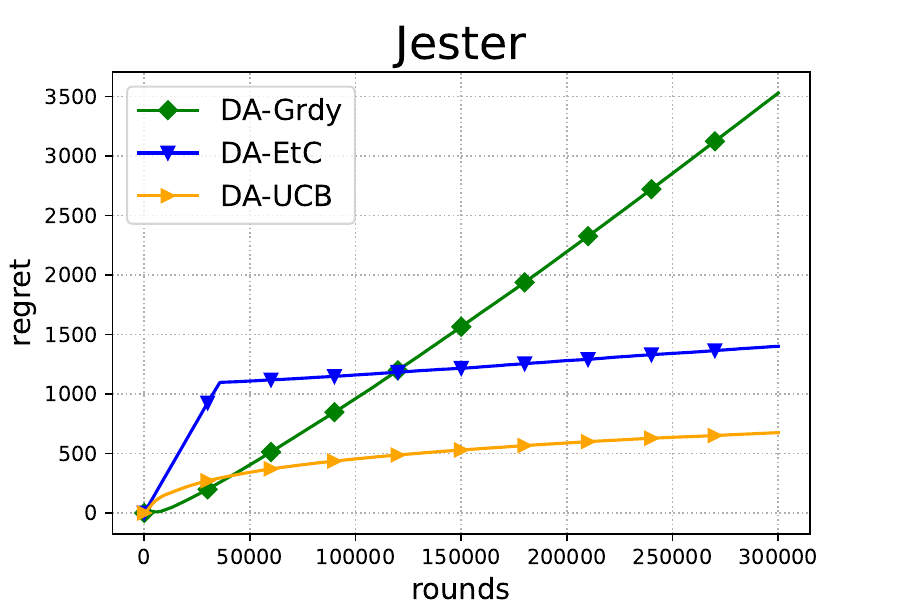}
    \subcaption{Jester (n=10 and m=50)}\label{fig:EtC_UCB2}
\end{minipage}
\begin{minipage}[t]{0.33\linewidth}
    \centering
    \includegraphics[keepaspectratio,width=\linewidth]{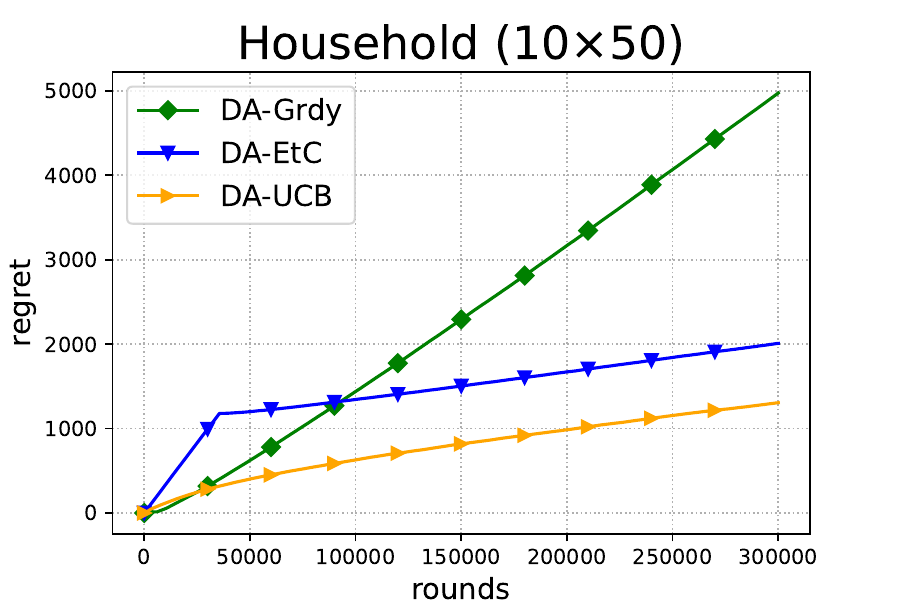}
    \subcaption{Household (n=10 and m=50)}\label{fig:EtC_UCB3}
\end{minipage}
\caption{
Regret of the three best algorithms, with the lines vertically enlarged to provide a clearer view of the data. The numbers presented are the same as Figure~\ref{fig:all-wide}. 
}\label{fig:all-narrow}
\end{figure*}

This section evaluates DA-EtC and DA-UCB in several synthetic and real datasets, \textit{Uniform}, \textit{Jester}, and \textit{Household}. 
For all data sets, we assume that the types of items are uniformly distributed. That is, $s_j = 1/m$ for all $j$. We set the objective as $B_i=1/n$ for all $i$.
The projection range of $\beta_i$ in DA-Iter is set to $[B_i/(1+0.95), (1+0.95)]$, which is wide enough to cover all datasets. 
The noise $\varepsilon_t\in\{0,1\}$ for realized utilities is determined according to a Bernoulli distribution, whose probability that the event $\varepsilon_t=1$ occurs is specified by each associated true value $v_{i,j}$, since it is normalized in $[0,1]$.
The length of the trial rounds $T_0$ for DA-EtC is set to $T^{2/3}(nm)^{1/3}$.
We adjust the numbers of agents ($n$) and item types ($m$), as well as the time horizon ($T$), according to the dataset used.
Figures~\ref{fig:all-wide} and \ref{fig:all-narrow} display the averaged amounts of regrets over 20 problem instances with different random seeds. 

\textbf{Benchmark algorithms:}
We have prepared three naive algorithms called \textit{Random}, \textit{UCB}, and \textit{DA-Grdy}. 
Random allocates each of the arriving items to agents at uniformly random. 
UCB allocates each one to the agent whose UCB value is the highest for the item type. UCB values, or each estimator $\hat{v}_{i,j}$, are initialized to one and are updated according to the sampled pairs of an agent and an item type as described in Line 8 in Algorithm 4. The UCB algorithm is designed to find an allocation that maximizes social welfare. 
As noted in Section~\ref{sec:problem setup}, the resulting allocation from UCB is expected to yield a lower NSW than DA-UCB since some of the agents cannot receive enough amounts of items under the SW maximizing allocation.
DA-Grdy does not stop updating the estimator $\hat{v}_{i,j}(t)$ at each round, unlike DA-EtC. 

\textbf{Datasets:}
\textit{Uniform}, \textit{Jester}
, and \textit{Household}
. 
First, to generate the Uniform dataset, we drew a set of values $\{v_{i,j}\}\in[0,1]^{n\times m}$ at uniformly random, where we set $n=10$, $m=10$. We ran algorithms up to time horizon $T=100000$. 

Second, Jester dataset was built for recommender systems and collaborative filtering studies~\citep{goldberg2001eigentaste}. We focus on the dataset that contains the ratings of 100 jokes by about 25000 individuals. Among them, 7200 individuals rate all 100 jokes (\url{https://eigentaste.berkeley.edu/dataset/jester_dataset_1_1.zip}~\citep{DBLP:journals/corr/abs-1901-06230}). We randomly select 10 out of 7200 individuals and 50 out of 100 jokes ($n=10$ and $m=50$). Time horizon $T$ is set to $300000$. 
Since the values of ratings lie between $-10$ and $10$, they are normalized to $[0,1]$ for our simulation. 

Finally, Household dataset consists of data from 2876 individuals regarding their estimated willingness-to-pay for 50 household items that were selected from an online review site~\citep{DBLP:journals/corr/abs-1901-06230}.\footnote{The exact dataset was provided by \citep{DBLP:journals/corr/abs-1901-06230} upon our request.} We herein randomly select 10 out of 2876 individuals ($n=10$ and $m=50$). Time horizon $T$ is set to $300000$ as well as Jester. 
Since the scores of willingness-to-pay lies between $0$ and $100$, we normalized it to $[0,1]$.

\textbf{Results:}

Figure~\ref{fig:all-wide} depicts Regret across the five algorithms for each of three datasets. Random and UCB are designed as a naive benchmark, and they incur significant regret compared to the DA-based algorithms. Notably, UCB performs worse than Random that is completely ignorant of the agents' preferences. This is because UCB maximizes social welfare (the sum of agents' utilities) instead of NSW (the product of them), and it often assigns items unequally, resulting in several agents receiving very small total utilities.
%
Figure~\ref{fig:all-narrow} depicts the same as Figure~\ref{fig:all-wide} except that the lines are vertically enlarged for a clearer view of the data. 
Apparently, DA-EtC and DA-UCB outperform DA-Grdy and achieve lower regret in the long run, although they do not in the short term. This is because DA-EtC and DA-UCB incur the cost for exploration before exploitation, whereas DA-Grdy attempts to exploit from the start.

DA-Grdy plug-ins empirical mean $\hat{v}_{i,j}(t)$ to DA-Iter for each round.
This algorithm suffers underexploration. For example, assume that $\hat{v}_{i,j}(t)$ is $0$ for some $i,j$ at some round $t$. DA is unlikely to assign future items of type $j$ to user $i$, which will prevent $\hat{v}_{i,j}$ from being updated. If such underestimation occurs with a probability of $\Theta(1)$ for some $(i,j)$, DA-Grdy can continue to approach an incorrect optimization, which results in a regret of $\Theta(T)$.

\section{Conclusion}

This paper has considered an online fair division problem where the values of items are unknown beforehand. In this problem, there is a natural notion of regret that measures how fast we can find the optimal allocation that asymptotically maximizes NSW. 

We proposed the algorithms to allocate items via dual averaging with its utility estimated from the past obtained utilities. We proposed two algorithms: DA-EtC and DA-UCB. The former is designed to feed i.i.d., data to DA. EA-EtC has the $\tilde{O}((nm)^{1/3}T^{2/3})$ regret, and the latter does not have such a regret bound but empirically performs better. Furthermore, we derived a $\tilde{O}((mT)^{1/2})$ regret lower bound irrespective of which algorithm is chosen. A version of DA-UCB called RDA-UCB achieves $\tilde{O}(\mathrm{poly}(n,m)\sqrt{T})$ regret, which is optimal with respect to $T$.

Our results call for subsequent research, including but not limited to: deriving a regret bound for the DA-UCB, and investigating the efficacy of structural models, such as linear models and factorized models.

\clearpage 

\bibliographystyle{abbrvnat}
\bibliography{reference_arxiv}

\begin{thebibliography}{29}
\providecommand{\natexlab}[1]{#1}
\providecommand{\url}[1]{\texttt{#1}}
\expandafter\ifx\csname urlstyle\endcsname\relax
  \providecommand{\doi}[1]{doi: #1}\else
  \providecommand{\doi}{doi: \begingroup \urlstyle{rm}\Url}\fi

\bibitem[Agarwal and Duchi(2013)]{agrawal_depdata}
A.~Agarwal and J.~C. Duchi.
\newblock The generalization ability of online algorithms for dependent data.
\newblock \emph{IEEE Transactions on Information Theory}, 59\penalty0
  (1):\penalty0 573--587, 2013.

\bibitem[Aleksandrov and Walsh(2020)]{aleksandrov:aaai:2020}
M.~Aleksandrov and T.~Walsh.
\newblock Online fair division: A survey.
\newblock \emph{Proceedings of the AAAI Conference on Artificial Intelligence},
  34:\penalty0 13557--13562, 2020.

\bibitem[Bart{\'{o}}k et~al.(2014)Bart{\'{o}}k, Foster, P{\'{a}}l, Rakhlin, and
  Szepesv{\'{a}}ri]{bartok2014}
G.~Bart{\'{o}}k, D.~P. Foster, D.~P{\'{a}}l, A.~Rakhlin, and
  C.~Szepesv{\'{a}}ri.
\newblock Partial monitoring - classification, regret bounds, and algorithms.
\newblock \emph{Mathematics of Operations Research}, 39\penalty0 (4):\penalty0
  967--997, 2014.

\bibitem[Bateni et~al.(2022)Bateni, Chen, Ciocan, and Mirrokni]{bateni2022fair}
M.~H. Bateni, Y.~Chen, D.~F. Ciocan, and V.~Mirrokni.
\newblock Fair resource allocation in a volatile marketplace.
\newblock \emph{Operations Research}, 70\penalty0 (1):\penalty0 288--308, 2022.

\bibitem[Bouveret and Lema^^c3^^aetre(2016)]{bouveret:comsoc:2016}
S.~Bouveret and M.~Lema^^c3^^aetre.
\newblock Efficiency and sequenceability in fair division of indivisible goods
  with additive preferences.
\newblock In \emph{Proceedings of the Sixth International Workshop on
  Computational Social Choice}, 2016.

\bibitem[Brams and Taylor(1996)]{brams:1996}
S.~J. Brams and A.~D. Taylor.
\newblock \emph{Fair division - from cake-cutting to dispute resolution.}
\newblock Cambridge University Press, 1996.

\bibitem[Budish(2011)]{budish:jpe:2011}
E.~Budish.
\newblock {The Combinatorial Assignment Problem: Approximate Competitive
  Equilibrium from Equal Incomes}.
\newblock \emph{Journal of Political Economy}, 119\penalty0 (6):\penalty0
  1061--1103, 2011.

\bibitem[Caragiannis et~al.(2019)Caragiannis, Kurokawa, Moulin, Procaccia,
  Shah, and Wang]{Caragiannis:1:2019}
I.~Caragiannis, D.~Kurokawa, H.~Moulin, A.~D. Procaccia, N.~Shah, and J.~Wang.
\newblock The unreasonable fairness of maximum nash welfare.
\newblock \emph{ACM Transactions on Economics and Computation (TEAC)},
  7\penalty0 (3):\penalty0 1--32, 2019.

\bibitem[Codenotti and Varadarajan(2007)]{codenotti:2007}
B.~Codenotti and K.~Varadarajan.
\newblock \emph{Computation of Market Equilibria by Convex Programming}, pages
  135--^^e2^^80^^93158.
\newblock Cambridge University Press, 2007.

\bibitem[Eisenberg and Gale(1959)]{eisenberg:ams:1959}
E.~Eisenberg and D.~Gale.
\newblock Consensus of subjective probabilities: The pari-mutuel method.
\newblock \emph{The Annals of Mathematical Statistics}, 30\penalty0
  (1):\penalty0 165--168, 1959.

\bibitem[Gao et~al.(2021)Gao, Peysakhovich, and Kroer]{gao:pace:2021}
Y.~Gao, A.~Peysakhovich, and C.~Kroer.
\newblock Online market equilibrium with application to fair division.
\newblock In M.~Ranzato, A.~Beygelzimer, Y.~Dauphin, P.~Liang, and J.~W.
  Vaughan, editors, \emph{Advances in Neural Information Processing Systems},
  volume~34, pages 27305--27318, 2021.

\bibitem[Goldberg et~al.(2001)Goldberg, Roeder, Gupta, and
  Perkins]{goldberg2001eigentaste}
K.~Goldberg, T.~Roeder, D.~Gupta, and C.~Perkins.
\newblock Eigentaste: A constant time collaborative filtering algorithm.
\newblock \emph{Information Retrieval}, 4:\penalty0 133--151, 2001.

\bibitem[Hazan and Levy(2014)]{hazan2014}
E.~Hazan and K.~Y. Levy.
\newblock Bandit convex optimization: Towards tight bounds.
\newblock In \emph{Advances in Neural Information Processing Systems 27: Annual
  Conference on Neural Information Processing Systems 2014}, pages 784--792,
  2014.

\bibitem[Jain and Vazirani(2010)]{GEB:2010}
K.~Jain and V.~V. Vazirani.
\newblock Eisenberg^^e2^^80^^93gale markets: Algorithms and game-theoretic
  properties.
\newblock \emph{Games and Economic Behavior}, 70:\penalty0 84--106, 2010.

\bibitem[Kash et~al.(2014)Kash, Procaccia, and Shah]{kash:jair:2014}
I.~Kash, A.~D. Procaccia, and N.~Shah.
\newblock No agent left behind: Dynamic fair division of multiple resources.
\newblock \emph{Journal of Artificial Intelligence Research}, 51\penalty0
  (1):\penalty0 579^^e2^^80^^93--603, 2014.

\bibitem[Kaufmann et~al.(2016)Kaufmann, Capp{\'{e}}, and Garivier]{kaufmann16}
E.~Kaufmann, O.~Capp{\'{e}}, and A.~Garivier.
\newblock On the complexity of best-arm identification in multi-armed bandit
  models.
\newblock \emph{Journal of Machine Learning Research}, 17:\penalty0 1:1--1:42,
  2016.

\bibitem[Komiyama et~al.(2015)Komiyama, Honda, and Nakagawa]{komiyama2017}
J.~Komiyama, J.~Honda, and H.~Nakagawa.
\newblock Regret lower bound and optimal algorithm in finite stochastic partial
  monitoring.
\newblock In C.~Cortes, N.~Lawrence, D.~Lee, M.~Sugiyama, and R.~Garnett,
  editors, \emph{Advances in Neural Information Processing Systems}, volume~28.
  Curran Associates, Inc., 2015.

\bibitem[Kroer et~al.(2021)Kroer, Peysakhovich, Sodomka, and
  Moses]{DBLP:journals/corr/abs-1901-06230}
C.~Kroer, A.~Peysakhovich, E.~Sodomka, and N.~E.~S. Moses.
\newblock Computing large market equilibria using abstractions.
\newblock \emph{Operations Research}, 70\penalty0 (1):\penalty0 329--351, 2021.

\bibitem[Lai and Robbins(1985)]{LAI1985}
T.~Lai and H.~Robbins.
\newblock Asymptotically efficient adaptive allocation rules.
\newblock \emph{Advances in Applied Mathematics}, 6\penalty0 (1):\penalty0
  4--22, 1985.

\bibitem[Liao et~al.(2022)Liao, Gao, and Kroer]{liao2022nonstationary}
L.~Liao, Y.~Gao, and C.~Kroer.
\newblock Nonstationary dual averaging and online fair allocation.
\newblock In A.~H. Oh, A.~Agarwal, D.~Belgrave, and K.~Cho, editors,
  \emph{Advances in Neural Information Processing Systems}, 2022.

\bibitem[Moulin(2003)]{moulin:2003}
H.~Moulin.
\newblock \emph{{Fair Division and Collective Welfare}}.
\newblock The MIT Press, 2003.

\bibitem[Nash(1950)]{nash:ecma:1950}
J.~F. Nash.
\newblock The bargaining problem.
\newblock \emph{Econometrica}, 18\penalty0 (2):\penalty0 155--162, 1950.

\bibitem[Procaccia and Wang(2014)]{procaccia:ec:2014}
A.~D. Procaccia and J.~Wang.
\newblock Fair enough: Guaranteeing approximate maximin shares.
\newblock In \emph{Proceedings of the Fifteenth ACM Conference on Economics and
  Computation}, pages 675--^^e2^^80^^93692, 2014.

\bibitem[Sinclair et~al.(2022)Sinclair, Banerjee, and
  Yu]{sinclair:sigmetric:2022}
S.~R. Sinclair, S.~Banerjee, and C.~L. Yu.
\newblock Sequential fair allocation: Achieving the optimal envy-efficiency
  tradeoff curve.
\newblock \emph{Operations Research}, 2022.
\newblock To appear.

\bibitem[Vazirani(2007)]{vazirani:2007}
V.~V. Vazirani.
\newblock \emph{Combinatorial Algorithms for Market Equilibria}, chapter~5,
  pages 103--^^e2^^80^^93134.
\newblock Cambridge University Press, 2007.

\bibitem[Vazirani(2012)]{vazirani:acm:2012}
V.~V. Vazirani.
\newblock The notion of a rational convex program, and an algorithm for the
  arrow-debreu nash bargaining game.
\newblock \emph{Journal of the ACM}, 59\penalty0 (2), 2012.

\bibitem[Vishnoi(2021)]{vishnoi:2021}
N.~K. Vishnoi.
\newblock \emph{Algorithms for Convex Optimization}.
\newblock Cambridge University Press, 2021.

\bibitem[Wang et~al.(2022)Wang, Sharma, Xu, Badam, Sun, Richardson, Chung, Chi,
  and Chen]{yuyan2022kdd}
Y.~Wang, M.~Sharma, C.~Xu, S.~Badam, Q.~Sun, L.~Richardson, L.~Chung, E.~H.
  Chi, and M.~Chen.
\newblock Surrogate for long-term user experience in recommender systems.
\newblock In \emph{Proceedings of the 28th ACM SIGKDD Conference on Knowledge
  Discovery and Data Mining}, page 4100^^e2^^80^^934109, 2022.

\bibitem[Xiao(2009)]{xiao_dual}
L.~Xiao.
\newblock Dual averaging method for regularized stochastic learning and online
  optimization.
\newblock In Y.~Bengio, D.~Schuurmans, J.~Lafferty, C.~Williams, and
  A.~Culotta, editors, \emph{Advances in Neural Information Processing
  Systems}, volume~22, 2009.

\end{thebibliography}

\onecolumn
\appendix\
\section{Notation table}
 \label{sec_notation}
Table~\ref{tbl_not} summarizes our notation.

\begin{table}[H]
\begin{center}
\caption{Major notation
}
\label{tbl_not}
\renewcommand{\arraystretch}{1.1} 
\begin{tabular}{lll} 
symbol & definition 
\\ \hline
$N$ & set of agents \\
$M$ &  set of the types of items\\
$v_{i,j}$&  ex ante (expected) value item of type $j$ for agent $i$\\
$T$ & time horizon (the number of rounds)\\
$t \in T$ & each round in $T$\\
$i(t) \in N$ & winner, or agent who is allocated the item at round $t$\\
$j(t) \in M$ & item type which arrives at round $t$ \\
$s_j \in \mathcal{S}$ &  probability distribution where item type $j$ is drawn\\
$u_i(t)$&  ex post (realized) utility of agent~$i$ at round $t$\\
$\epsilon_t$& sub-Gaussian random variable
with its radius $\sigma^2$\\
$U_i(T)$& cumulative utility of agent~$i$ in $T$ rounds\\
$B_i$&  priority or per-period budget rate given to agent~$i$\\
$x_{i,j}$& fraction of item type~$j$ allocated to agent~$i$ in the EG program (Eq.~\eqref{ineq_ONSW}) \\
$p_j$& price of item type~$j$ in $D_{EG}$\\
$\beta_i$&  multiplier in DA-Iter (Algorithm \ref{DAiter}) \\
$\hat{v}^{\DA}_{i,j(t)}$&  argument value $v_{i,j}$ in DA-Iter (Algorithm \ref{DAiter})\\
$\bar{u}^{\DA}_i$& mean utility of $u^{\DA}_i(t)$ (Algorithm \ref{DAiter})\\
$T_0$& exploration rounds in DA-EtC (Algorithm \ref{DA-EtC})\\
$\hat{v}_{i,j}^{\EtC}$& Estimator in DA-EtC (Algorithm \ref{DA-EtC})\\
$\vucb_{i,j}$&  UCB value in DA-UCB (Algorithm \ref{DA-UCB})\\
$N_{i,j}(t)$& number of times agent~$i$ received an item of type~$j$ up to round $t-1$ \\
$u^{*,\DA}$& solution of EG with estimators $\{\hat{v}_{i,j}^{\EtC}\}_{i,j}$\\
$u^{*,\true}$ &  utility of EG problem with $v_{i,j}$ \\
$u^{*,\RUCB}$ & Defined in Eq.~\eqref{ineq_rucb_meanopt}.
\\
$u^{*,\RUCB(k)}$ & utility of EG problem with corresponding UCB values of $k$th instance of DA
\\
$u^{\DA,\true}$ & mean utility that agents receive when we run the algorithm (Eq.~\eqref{ineq_meanutility}) \\
$u^{\DA,\RUCB}$ & Defined in Eq.~\eqref{ineq_rucb_meanda}  
\\
$u^{\DA,\RUCB(k)}$ & mean utility of the $k$th instance of Dual Averaging. 
\\
$\hat{v}_{i,j}^{\EtC}$& Estimator in DA-EtC (Algorithm \ref{DA-EtC})\\
$\vucb_{i,j}(t)$&  UCB value of round $t$ in DA-UCB (Algorithm \ref{DA-UCB})\\
$\hat{v}_{i,j}^{\RUCB}(t)$ &  UCB value of round $t$ in RDA-UCB (Algorithm \ref{RDA-UCB}) \\
$T_k$ & number of rounds in which $k$th instance of DA was run \\

\end{tabular}
\end{center}
\end{table}

\section{Relation with Bandit Convex Optimization}
\label{sec_bandit_convex}

This section discusses the online optimization (i.e., regret minimization) in this paper and bandit convex optimization (BCO) \cite{hazan2014}, which is known to have an $O(\sqrt{T})$ bound under some assumptions.
\begin{itemize}
\item BCO is more challenging than our optimization in the sense that the loss function is given adversarially. This means that stochastic bandit algorithms, such as UCB, cannot be directly applied to BCO.
\item On the contrary, our regret with respect to fair division optimizes latent parameters ($\beta_i$) that depend on the data sequence. This is challenging because we cannot directly observe the loss for each round.
\end{itemize}
In summary, the difficulty of BCO and our setting cannot be directly compared. Therefore, our bounds ($\tilde{O}(T^{2/3})$ upper bound and $\tilde{\Omega}(T^{1/2})$ lower bound) are nontrivial.

\section{Proofs on DA-EtC}

\subsection{Proof of Theorem \ref{thm_reglower}}
\begin{proof}[Proof of Theorem \ref{thm_reglower}]
We consider the case in which the type of items is equally distributed; $s_j = 1/m$ for all $j \in [M]$, and $B_i = 1/n$ for all $i \in [N]$. 
We consider the case where all feedback is binary.
In this proof, we use the term ``model'' to denote a value matrix $\{v_{i,j}\}_{i,j} \in \mathbb{R}^{n \times m}$.
We consider the following classes of models, $v:\ \forall_{i,j}\,v_{i,j} \in \{1/2, 1/2 (1 + \sqrt{m/T})\}^{nm}$.
We call $v^{(0)}:\forall_{i,j}\,v_{i,j}^{(0)} = 1/2$ the \textit{base} model.

The following lemma, which is a version of Lemma 19 in ~\cite{kaufmann16}, is used during the proof.
\begin{lemma}{\rm (lower bound on any event)}\label{lem_kldiv}
Let $v^{(1)}, v^{(2)}$ be two models. 
Let $\Ex_{v^{(1)}}, \Ex_{v^{(2)}}$ be the corresponding expectations and $\mathbb{P}_{v^{(1)}}, \mathbb{P}_{v^{(2)}}$ be the corresponding probabilities, respectively. 
Then, the following inequality holds for any event $\mE$.
\begin{equation}
\sum_{i\in[n], j\in[m]} 
\Ex_{v^{(1)}}\left[
N_{i,j}(T+1)
\right]
d(v_{i,j}^{(1)}, v_{i,j}^{(2)})
\ge
d(\mathbb{P}_{v^{(1)}}(\mE), \mathbb{P}_{v^{(2)}}(\mE)),
\end{equation}
where $d(p,q) = p \log(p/q) + (1-p) \log((1-p)/(1-q))$ is the KL divergence between two Bernoulli distributions.
\end{lemma}

Let $j_i = \argmin_{j'} \Ex_{v^{(0)}}[N_{i,j'}(T+1)]$ be the type of item that agent $i$ receives the least number of times during $T$ under the base model. 
By definition, $\sum_i N_{i,j_i}(T+1) \le T/m$. Consider another model $v^{(a)}$ where $v_{i,j_i} = (1 + \sqrt{m/T})/2$ for each $i \in N$ and $v_{i,k} = 1/2$ for $k \ne j_i$. We have
\begin{align}
\lefteqn{
\sum_{i\in[n], j\in[m]} 
\Ex_{v^{(0)}}\left[
N_{i,j}(T+1)
\right]
d(v_{i,j}^{(0)}, v_{i,j}^{(a)})
}\\
&\le (T/m) d(1/2, 1/2 (1 + \sqrt{m/T}) )\\
&\le (T/m) \times O(m/T) \text{\ \ \ \ (by $d(1/2, 1/2+\alpha) = O(\alpha^2)$)}\\
&= O(1). \label{ineq_divconst}
\end{align}

Consider the event
\[
\mE = \left\{
\sum_i N_{i,j_i}(T+1) \le 2T/m
\right\}.
\]
By definition, $\Prob_{v^{(0)}}[\mE] \ge 1/2$. Lemma \eqref{lem_kldiv} and Eq.~\eqref{ineq_divconst} implies that 
\[
d(\Prob_{v^{(0)}}[\mE], \Prob_{v^{(a)}}[\mE]) = O(1),
\]
which implies that
\[
\Prob_{v^{(a)}}[\mE] = \Omega(1).
\]
Under event $\mE$ on the alternative model, the regret is at least $\Omega(\sqrt{mT})$, which completes the proof.

\end{proof} 

\subsection{Proof of \thmref{thm_etc}}
\begin{proof}[Proof of \thmref{thm_etc}]

Let $l' = 2l, h' = h/2$.
EtC uniformly explores during the first $T_0$ rounds, and the estimator $\hat{v}_{i,j} = \hat{v}_{i,j}(T_0+1)$ is based on $N_{i,j}(T_0+1) \approx T_0/(nm)$ samples. 
Some care is needed because $N_{i,j}(T_0+1)$ itself is a random variable.
Let 
\begin{align}
\mA &= 
\bigcap_{i,j}
\left\{
N_{i,j}(T_0+1) 
\ge \frac{T_0}{2 nm}
\right\}\\
\mB &= 
\bigcap_{i,j}
\left\{
\left|
\vep_{i,j}(T_0+1) - v_{i,j}
\right| \le \sqrt{\frac{8\sigma^2 nm \log (nmT)}{T_0}}
\right\}.
\end{align}
Since $N_{i,j}(T_0+1)$ is a sum of $T_0$ binary random variables with its mean $1/(nm)$ (i.e., item is of type $j$, and agent $i$ won it), by the multiplicative Chernoff inequality, 
\[
N_{i,j}(T_0+1) 
< \frac{T_0}{2nm}
\] 
holds with probability at most $\exp(-T_0/(8nm)) < 1/(nmT)$, and Event $\mA$ holds with probability at least $1- 1/T$ by considering the union bound of it over $i \in N, j \in M$.

Second, we show that $\mB$ holds with a high-probability given $\mA$.
Given $n'$ i.i.d. samples with its mean $v_{i,j}$ and sub-Gaussian radius $\sigma$, we have
\[
|\vep_{i,j}(T_0+1) - v_{i,j}|
\le 
\sqrt{\frac{4 \sigma^2 \log(nmT)}{{n'}}}
\]
with probability at least $1-1/(nmT)^2$. Union bound of this over possible random value $n' = N_{i,j}(T_0+1) \in [T]$ yields 
\[
|\vep_{i,j}(T_0+1) - v_{i,j}|
\le 
\sqrt{\frac{8 \sigma^2 nm\log(nmT)}{T_0}}
\,\biggl|\, N_{i,j}(T_0+1) \ge \frac{T_0}{2nm}
\]
and taking its union bound over $i\in N,j\in M$ yields $\mB$ that holds with probability at least $1 - 1/(nmT)$ given $\mA$.
In the following steps, we assume $\mB$ because the probability that $\mB$ does not hold is $O(1/T)$, and the regret in this is at most $O(T) \times O(1/T) = O(1)$, which is negligible. 

Assuming that $\mB$ holds at the end of round $T_0$, we bound the regret. 
\begin{align}
\lefteqn{
\frac{\Regret(T)}{T} 
}\nn
&= \prod_i (u^{*,\true}_i)^{B_i} 
-
\frac{1}{T}
\prod_{i \in N} U_i(T)^{B_i}\nn
&\le \prod_i (u^{*,\true}_i)^{B_i} 
-
\frac{1}{T}
\prod_{i \in N} ((T-T_0) \bar{u}_i^{\EtC})^{B_i}
\text{\ \ \ \ \ \ \ (the sum of utility in the first $T_0$ rounds is non-negative)}
\nn
&= \prod_i (u^{*,\true}_i)^{B_i} 
-
\frac{T-T_0}{T}
\prod_{i \in N} (\bar{u}_i^{\EtC})^{B_i}
\text{\ \ \ \ \ \ \ (by $\sum_i B_i = 1$)}
\nn
&=
\frac{T_0}{T}
\prod_i (u^{*,\true}_i)^{B_i} 
+
\frac{T-T_0}{T}\left(
\prod_i (u^{*,\true}_i)^{B_i}
-
\prod_i (\bar{u}_i^{\EtC})^{B_i}
\right)
\nn
&\le
\frac{T_0}{T} \prod_i (u^{*,\true}_i)^{B_i} 
+
\left|
\prod_i (u^{*,\true}_i)^{B_i} 
- 
\prod_i (\bar{u}_i^{\EtC})^{B_i}
\right|.
\end{align}
Here, the second term is bounded as
\begin{align}
\prod_i (u^{*,\true}_i)^{B_i} 
- 
\prod_i (\bar{u}_i^{\EtC})^{B_i}
&= \left(\prod_i (u^{*,\true}_i)^{B_i} - \prod_i (u_i^{*,\DA})^{B_i}\right) 
+ \left(\prod_i (u_i^{*,\DA})^{B_i} - \prod_i (\bar{u}_i^{\DA})^{B_i}\right) \\
&\ \ \ \ \ \ \ \ \ \ + \left(\prod_i (\bar{u}_i^{\DA})^{B_i} - \prod_i (\bar{u}_i^{\EtC})^{B_i}\right).    
\end{align}

From the definition of $u_i^{\ast, \DA}$, for any $\{x_{i,j}\}$, $\prod_i (u_i^{\ast, \DA})^{B_i} \geq \prod_i \left(\sum_{j\in M}s_j\hat{v}_{i,j}(T_0 + 1)x_{i,j}\right)^{B_i}$ holds.
Hence, we have
\begin{equation*}
\prod_i (u^{*,\true}_i)^{B_i} - \prod_i (u_i^{*,\DA})^{B_i}
\leq \prod_i (u^{*,\true}_i)^{B_i} - \prod_i (\hat{u}_i^{\ast, \DA})^{B_i},
\end{equation*}
where $\hat{u}_i^{\ast, \DA}=\sum_{j\in M} s_j \hat{v}_{i,j}(T_0 + 1)x_{i,j}^{\ast}$ and $\{x_{i,j}^{\ast}\}$ be the solution of the optimization (Eq.~\eqref{ineq_ONSW}) with true $\{v_{i,j}\}$.
On the other hand, under $\mB$,
\begin{align*}
    |u_i^{\ast} - \hat{u}_i^{\ast, \DA}| = \left|\sum_{j\in M}s_j x_{i,j}^{\ast}(v_{i,j} - \hat{v}_{i,j}(T_0 + 1))\right| \leq \sum_{j\in M}s_jx_{i,j}^{\ast}\left|v_{i,j}-\hat{v}_{i,j}(T_0 + 1)\right| \leq \tilde{O}\left(\sqrt{\frac{nm}{T_0}}\right)
\end{align*}
Combining these inequalities and Lemma \ref{lem_da_sensitivity}, under the assumption that $u^{*,\true}_i = \Theta(1)$,
\begin{equation}
\prod_i (u^{*,\true}_i)^{B_i} - \prod_i (u_i^{*,\DA})^{B_i} \leq \tilde{O}\left(\sqrt{\frac{nm}{T_0}}\right).
\end{equation}
Moreover, 
Event $\mB$ and assumptions imply that $l'/2 \le \hat{v}_{i,j}(T_0+1) \le h' + l'/2 \le 2h'$, and thus 
Lemma \ref{lem_da_nsw} with $l=l'/2, h=2h'$ states that
\begin{equation}
\Ex\left[
\prod_i (u_i^{*,\DA})^{B_i} - \prod (\bar{u}_i^{\DA})^{B_i}
\right]
\le \tilde{O}\left(n \CDA\sqrt{\frac{1}{T}}\right),
\end{equation}
where we consider $l',h'>0$ to be constants.

    DA gives at least $\Omega(T/(\sum_i (4h'/l')^2)) = \Omega((T/n) \times 1) = \Omega(T/n)$ items to each agent because if agent $i_2$ receives $(2h'/(l'/2))^2$ times more items than agent $i_1$, then the ratio $\beta_{i_2}/\beta_{i_1}$ is at least $(2h'/(l'/2))$, and agent $i_1$ is prioritized in the next allocation to agent $i_2$ no matter what type of item is.
Using this fact, we have $\bar{u}_i^{\DA} = \Theta(1)$. Applying a concentration inequality to the samples during the exploitation phase, with probability at least $1 - 1/T$, we have
$\bar{u}_i^{\EtC} = \Theta(1), |\bar{u}_i^{\EtC} - \bar{u}_i^{\DA}| = \tilde{O}(\sqrt{nm/T_0})$. Using this and Lemma \ref{lem_da_sensitivity}, we have 
\begin{equation}
(\prod_i (\bar{u}_i^{\DA})^{B_i}) - \prod_i (\bar{u}_i^{\EtC})^{B_i})
=
\tilde{O}\left(\sqrt{\frac{nm}{T_0}}\right).
\end{equation}

In summary,
\begin{align}
\Ex\left[
\frac{\Regret(T)}{T} 
\right] 
= \frac{T_0}{T} 
+ \tilde{O}\left(\sqrt{\frac{nm}{T_0}}\right)
 + n \CDA\sqrt{\frac{1}{T}}.
\end{align}

\end{proof}

\subsection{Additional Lemmas on DA-EtC}

\begin{lemma}\label{lem_da_sensitivity}
Let two vectors $u^{(1)} = (u^{(1)}_1,\dots,u^{(1)}_n)$ and $u^{(2)} = (u^{(2)}_1,\dots,u^{(2)}_n)$ be such that $u^{(1)}_i = \Theta(1)$ an $u^{(2)}_i \ge u^{(1)}_i (1 - r_i)$, and $r = \max_i r_i$. 
Then,
\begin{equation}
\prod_i (u^{(1)}_i)^{B_i} - \prod_i (u_i^{(2)})^{B_i}
\le 
O\left(r\right).
\end{equation}
\end{lemma}
\begin{proof}[Proof of Lemma \ref{lem_da_sensitivity}]
\begin{align}
\prod_i (u_i^{(1)})^{B_i} - \prod_i (u_i^{(2)})^{B_i}
&= \prod_i (u^{(1)}_i)^{B_i} - \prod_i \left\{u^{(1)}_i (1 - r_i) \right\}^{B_i} \\
&\le \prod_i (u^{(1)}_i)^{B_i} - \prod_i \left\{u^{(1)}_i (1 - \max_{i'} r_{i'})\right\}^{B_i} \\
&= \max_{i'} r_{i'} \prod_i (u^{(1)}_i)^{B_i}.
\end{align}
\end{proof} 

\begin{lemma}\label{lem_da_nsw}
Assume that $T_0 = o(T)$. 
Then, the following inequality holds:
\begin{equation}
\Ex\left[
\prod_i (u_i^{*,\DA})^{B_i} - \prod_i (\bar{u}_i^{\DA})^{B_i}
\right]
\le
\tilde{O}\left(n \CDA \sqrt{\frac{1}{T}}\right).
\end{equation}
\end{lemma}

\begin{proof}[Proof of Lemma \ref{lem_da_nsw}]
Remember that $\bar{u}_i^{\DA}$ is the regret-per-round during the exploitation rounds. Let $T' = T - T_0 = \Theta(T)$. 
In view of DA, it is an online learning with $T'$ rounds where the value of each item is $\{\vep_{i, j(t)}\}_i$.

Lemma \ref{lem_dabound} implies that 
\[
\Ex\left[
(u^{*,\DA}_i - \bar{u}_i^{\DA})^2
\right]
\le \frac{\CDA \log T'}{T'},
\]
and Markov's inequality implies that
\begin{align*}
\Pr[|u^{*,\DA}_i - \bar{u}_i^{\DA}| \ge \eps' ] 
\le \frac{1}{\eps'^2}\frac{\CDA\log T'}{T'},
\end{align*}
and by letting $\eps' = \eps u^{*,\DA}_i$ we have
\begin{align}\label{ineq_markov}
\Pr\left[
\frac{|u^{*,\DA}_i - \bar{u}_i^{\DA}|}{u^{*,\DA}_i} \ge \eps
\right] 
\le \frac{1}{(u^{*,\DA}_i)^2 \eps^2}\frac{\CDA\log T'}{T'},
\end{align}
Taking union bound of Eq.~\eqref{ineq_markov} over $i \in N$, we have
\begin{align}\label{ineq_markov_R}
\Pr\left[
\max_i \frac{|u^{*,\DA}_i - \bar{u}_i^{\DA}|}{u^{*,\DA}_i} \ge \eps
\right] 
\le \frac{n}{\min_i (u^{*,\DA}_i)^2 \eps^2}\frac{\CDA\log T'}{T'}.
\end{align}

Here, letting 
\[
R := \max_i \frac{u^{*,\DA}_i - \bar{u}_i^{\DA}}{u^{*,\DA}_i} \leq 1,
\]
then $ \bar{u}_i^{\DA}  \ge (1-R)u^{*,\DA}_i$ for all $i$ and we have 
\begin{align*}
\prod_i (u_i^{*,\DA})^{B_i} - \prod_i (\bar{u}_i^{\DA})^{B_i}
&\le \prod_i (u_i^{*,\DA})^{B_i} - \prod_i ((1-R)u_i^{*,\DA})^{B_i}\\
&\le R \prod_i  (u_i^{*,\DA})^{B_i}. 
\text{\ \ \ \ (by $\prod_i (x)^{B_i} = x$)}
\end{align*}
Using this, we have
\begin{align*}
\lefteqn{
\Ex\left[
\prod_i (u_i^{*,\DA})^{B_i} - \prod_i (\bar{u}_i^{\DA})^{B_i}
\right]
}\\
&\le \Ex\left[R \prod_i  (u_i^{*,\DA})^{B_i}\right]\\
&\le \int_{1/\sqrt{T'}}^1 \prod_i(u_i^{*,\DA})^{B_i} \Pr[R \ge x] dx + \sqrt{\frac{1}{T'}} \prod_i (u_i^{*,\DA})^{B_i}\\
&= \int_{1/\sqrt{T'}}^1 \prod_i(u_i^{*,\DA})^{B_i} \Pr[R \ge x] dx + O\left(\sqrt{\frac{1}{T'}}\right)\\
&\le \int_{1/\sqrt{T'}}^1 \prod_i(u_i^{*,\DA})^{B_i} \frac{n}{\min_i (u^{*,\DA}_i)^2 x^2}\frac{\CDA\log T'}{T'} dx 
+ O\left(\sqrt{\frac{1}{T'}}\right)
\text{\ \ \ \ \ (by Eq.~\eqref{ineq_markov_R})}\\
&= (u_i^{*,\DA})^{B_i} \frac{n}{\min_i (u^{*,\DA}_i)^2}\frac{\CDA\log T'}{T'} [-1/x]_{1/\sqrt{T'}}^1 + O\left(\sqrt{\frac{1}{T'}}\right)\\
&\leq (u_i^{*,\DA})^{B_i} \frac{n}{\min_i (u^{*,\DA}_i)^2}\frac{\CDA\log T'}{\sqrt{T'}} + O\left(\sqrt{\frac{1}{T'}}\right) \\
&= \tilde{O}\left(n \CDA \sqrt{\frac{1}{T'}}\right). 
\end{align*}
\end{proof} 

\section{Proofs on RDA-UCB}
\subsection{Proof of Theorem \ref{thm_ucb_reset}}

\begin{proof}[Proof of Theorem \ref{thm_ucb_reset}] 
First, let $u_i^{*,\true}$ be the utility of EG problem with $v_{i,j}$. 
Let 
\begin{equation}\label{ineq_meanutility}
u_i^{\DA,\true} = U_i(T)/T 
\end{equation}
be the empirical mean utility that each agent receives when we run the algorithm.
Remember that RDA-UCB utilizes several instances of DA. Let $y$ be the number of the DA instances that appear in RDA-UCB. We use index $k=1,2,\dots,y$ to represent each instance of DA. Let $T_k$ be the number of rounds where each DA was run. Note that each $T_k$, as well as $y$, are the random variables.
Let $u_i^{*,\RUCB}, u_i^{\DA,\RUCB}$ be:
\begin{align}
\label{ineq_rucb_meanopt}
u_i^{*,\RUCB}&=\frac{1}{T}(T_1u_i^{*,\RUCB(1)}+T_2u_i^{*,\RUCB(2)}+\dots+T_yu_i^{*,\RUCB(y)}),\\
u_i^{\DA,\RUCB}&=\frac{1}{T}\sum_{t=1}^T \Ind[I(t)=i] \vucb_{i(t),j(t),N_{i,j(t)}(s)}\\
&=\frac{1}{T}(T_1u_i^{\DA,\RUCB(1)}+T_2u_i^{\DA,\RUCB(1)}+\dots+T_yu_i^{\DA,\RUCB(y)}).
\label{ineq_rucb_meanda}
\end{align}
Namely, The value $u_i^{*,\RUCB}$ indicates the mean true utility and the random variable $u_i^{\DA,\RUCB}$ indicates the mean utility over DA instances.

Moreover, let $\vrucb_{i,j,n}$ is the empirical estimate of $v_{i,j}$ with $n$ samples and two events be
\begin{equation}\label{ineq_ucbreset_upward}
\mG_1 = 
\bigcap_{i,j,n} \{\vrucb_{i,j,n} \ge v_{i,j}\},
\end{equation}
\begin{equation}\label{ineq_ucbreset_lower}
\mG_2 = 
\bigcap_i 
\left\{
\sum_{t=1}^T 
\vrucb_{i(t),j(t),N_{i,j(t)}(t)}
\le \sum_{t=1}^T 
(v_{i(t),j(t)} + \eps_t)
 + C_3 nm\sqrt{T}
\right\},
\end{equation}
where $C_3 = \tilde{O}(1)$ is defined later in Eq.~\eqref{ineq_thirdterm_bound}.

From lemma \ref{lem_ucbreset_rewards}, $\mG := \mG_1 \cap \mG_2$ holds with probability at least $1 - O(1/T)$.

We have 
\begin{align}
\Ind[\mG] 
\frac{\Regret(T)}{T}
&=
\Ind[\mG](\prod_i (u^{*,\true}_i)^{B_i} - \prod_i (u_i^{*,\RUCB})^{B_i}) +  
\Ind[\mG](\prod_i (u_i^{*,\RUCB})^{B_i}) - \prod_i (u_i^{\DA,\RUCB})^{B_i}) \nn
&+ \Ind[\mG](\prod_i (u_i^{\DA,\RUCB})^{B_i}) - \prod_i (u_i^{\DA,\true})^{B_i})
\nn
&\le  \Ind[\mG_1 \cap \mG_2](\prod_i (u^{*,\true}_i)^{B_i} - \prod_i (u_i^{*,\RUCB})^{B_i}) +  
|(\prod_i (u_i^{*,\RUCB})^{B_i}) - \prod_i (u_i^{\DA,\RUCB})^{B_i})| \nn
&+ \Ind[\mG_1 \cap \mG_2]|(\prod_i (u_i^{\DA,\RUCB})^{B_i}) - \prod_i (u_i^{\DA,\true})^{B_i})|
\nn
&\le 
0 
+ |(\prod_i (u_i^{*,\RUCB})^{B_i}) - \prod_i (u_i^{\DA,\RUCB})^{B_i})| +
\tilde{O}(\frac{nm\sqrt{T}}{T}),
\label{ineq_oneandthree}
\end{align}
where we have used Lemmas \ref{lem_UCB_proof1} and \ref{lem_UCB_proof3} in the last transformation.

Therefore, the regret bound is:
\begin{align}
\Ex\left[
\Regret(T) 
\right]
&= 
\Ex\left[\Ind[\mG] \Regret(T) \right]
+
\Ex\left[\Ind[\mG^c] \Regret(T) \right]
\\
&\le
\Ex\left[\Ind[\mG] \Regret(T) \right]
+
T \Pr\left[\mG^c\right]
\\
&\le
\Ex\left[\Ind[\mG] \Regret(T) \right]
+
T \times \tilde{O}(\frac{1}{T})
\text{\ \ \ \ (by Lemma \ref{lem_ucbreset_rewards})}\\
&\le \Ex\left[T\left|\prod_i (u_i^{*,\RUCB})^{B_i}) - \prod_i (u_i^{\DA,\RUCB})^{B_i}\right| \right]
+ \tilde{O}(nm\sqrt{T})
\text{\ \ \ \ (by Eq.~\eqref{ineq_oneandthree})}\label{ineq_twoleft}
\end{align}

Lemma \ref{lemma_secondterm_newmain} bounds the first term on RHS of Eq.~\eqref{ineq_twoleft} bounded as
\begin{equation}
\Ex\left[T\left|\prod_i (u_i^{*,\RUCB})^{B_i}) - \prod_i (u_i^{\DA,\RUCB})^{B_i}\right| \right] = 
\tilde{O}\left( n^4 m^2 \sqrt{T} \right),
\end{equation}
which completes the proof.
\end{proof} 

\subsection{Lemmas on the probability of event $\mG$}

\begin{lemma}\label{lem_ucbreset_rewards}
Event $\mG_1 \cap \mG_2$ holds with probability at least $1 - O(1/T)$.
\end{lemma}
\begin{proof}[Proof of Lemma \ref{lem_ucbreset_rewards}]
Subgaussian concentration inequality implies
\begin{equation}\label{ineq_subgauss}
\bigcap_{i,j,n}
\left|
\vrucb_{i,j,n}
-
v_{i,j}
\right| \le \sigma 
\sqrt{\frac{2 \log(T^2)}{n}}
\end{equation}
holds with probability at least $1-2/T^2 \times nmT = 1 - O(1/T)$.
Eq.~\eqref{ineq_ucbreset_upward} easily follows from the Eq.~\eqref{ineq_subgauss}.
Moreover,
\begin{align}
\lefteqn{
\sum_{t=1}^T (
\vrucb_{i(t),j(t),N_{i(t),j(t)}(t)}
-
(v_{i(t),j(t)}+\eps_t)
)
}\\
&\le 
2h + 
\sum_{i,j}
\sum_{p=1}^{\log_2 T} 
2^p
(
\vrucb_{i,j,2^p}
-
v_{i,j}
)+
|\sum_t \eps_t|
\\
&\text{\ \ \ \ (by the fact that RDA-UCB resets once allocation $(i,j)$ reaches $N_i(t)=2,2^2,2^3,\dots,$)}\\
&\le 
2h + 
\sum_{i,j}
\sum_{p=1}^{\log_2 T} 
2^p
(
\vrucb_{i,j,2^p}
-
v_{i,j}
)+
\sigma \sqrt{2 \log(T^2)}
\sqrt{T}
\text{\ \ \ \ (by subgaussian concentration inequality on $\sum_t \eps(t)$)}\\
&\le 
2h + 
nm \sum_{p=1}^{\log_2 T} 
2^p 
\times 
\sigma 
\sqrt{\frac{2 \log(T^2)}{2^p}}
+
\tilde{O}(\sqrt{T})
\text{\ \ \ \ (by Eq.~\eqref{ineq_subgauss})}\\
&=:
C_3
nm 
\sqrt{T}
.\label{ineq_thirdterm_bound}
\end{align}
\end{proof} 

\subsection{Lemmas on the first term}
\begin{lemma}\label{lem_UCB_proof1}
for any agent i $\in N$, with $\mG_1 \cap \mG_2$ from lemma \ref{lem_ucbreset_rewards},
\begin{align}\label{ine_UCB_proof1}
\Ind[\mG_1 \cap \mG_2](\prod_i (u^{*,\true}_i)^{B_i} - \prod_i (u_i^{*,\RUCB})^{B_i}) \le 0.
\end{align}
\end{lemma}

\begin{proof}[Proof of Lemma \ref{lem_UCB_proof1}]
$\\$
Let $T_k$ be the number of rounds in which $k$th instance of DA was run and $u^{*,\RUCB(k)}$ be the mean utility of Dual Averaging in the $k$th instance of DA.
Let $y$ be the index of the last instance of DA.

Recall that
\begin{align}\label{ineq_ucbprod}
u_i^{*,\RUCB}
=
\frac{1}{T}(T_1u_i^{*,\RUCB(1)}+T_2u_i^{*,\RUCB(2)}+\dots+T_yu_i^{*,\RUCB(y)}).
\end{align}

Lemma \ref{lem_convexprod} implies that the region 
\[
\mathcal{U} = \left\{(u_i)_{i\in[N]} \in \mathbb{R}^N \,\biggl|\, (\prod_i (u^{*,\true}_i)^{B_i} \le (\prod_i (u_i)^{B_i})\right\}
\]
is convex. Event $\mG_1$ implies that $u^{*,\RUCB(k)}$ for each $k$ lies in $\mathcal{U}$. 
Since $u^{*,\RUCB}$ is a non-negative and Affine combination of elements in $\mathcal{U}$, $u^{*,\RUCB} \in \mathcal{U}$, which is Eq.~\eqref{ine_UCB_proof1}.
\end{proof}

\begin{lemma}{\rm (Convexity of region)}\label{lem_convexprod}
Let $v, v^1, v^2 \in (\Real^+)^n$ be $n$-dimensional vectors. If $\prod_i v^1_i, \prod_i v^2_i \ge \prod_i v_i$, then for any $p \in [0, 1]$, $\prod_i (p v^1_i + (1-p) v^2_i) \ge \prod_i v_i$.
\end{lemma}
\begin{proof}
Let $c = \prod_i v_i$. 
Weighted inequality of arithmetic and geometric means states that 
\begin{align}
p v^1_i + (1-p) v^2_i 
\ge (v^1_i)^p (v^2_i)^{1-p}
\end{align}
for each $i$, and thus
\begin{equation}
\prod_i (p v^1_i + (1-p) v^2_i)
\ge \left(\prod_i v^1_i\right)^p \left(\prod_i v^2_i\right)^{1-p} \ge
\left(\prod_i v_i\right)^p \left(\prod_i v_i\right)^{1-p}
= \prod_i v_i.
\end{equation}
\end{proof}

\subsection{Lemmas on the second term}

The structure of this section is as follows. Section \ref{subsec_rucb_singleda} shows the results for each instance of DA. Section \ref{subsec_rucb_multipleda} uses the results for RDA-UCB, which uses multiple instances of DA. 
By using these lemmas, Section \ref{subsec_termtwo_mainlemmas} bounds the second term in the main proof.

\subsubsection{Auxiliary lemmas on a single instance of DA}\label{subsec_rucb_singleda}

This section introduces lemmas that apply for each instance of DA. For ease of notation, we drop the index $k$ of the instances in the context it is clear. For example, $\beta^*$ indicates the optimal solution of the dual EG problem for the $k$th instance of DA. In this optimal solution, the value matrix is the corresponding UCB values  $v = \vucbclip_{i,j}(t_k)$, where $t_k$ is the first round of the $k$th instance.
 Lemmas \ref{lem:distance_upper_bound}--\ref{lem_UCB_proof2}) are used to derive Lemma \ref{lem:utility_bound}, and Lemma \ref{lem:utility_bound} is used in the subsequent lemmas.
Lemmas \ref{lem:distance_upper_bound}--\ref{lem_UCB_proof2} are the version of the similar results in \cite{xiao_dual} (Theorem 1(b) therein) that are tailored for our version of DA. During these lemmas, we use the notation of \cite{xiao_dual}. 
Namely, for a time step $\tau$ in view of DA, let $f(\beta, j(\tau))$ be $\max_{i\in N} v_{i,j(\tau)}\beta_i$ and $\Psi(\beta)$ be $-\sum_{i\in N}B_i\log \beta_i$.
Let $R_t(\beta)=\sum_{\tau=1}^t\left(f(\beta^{\tau}, j({\tau})) - \Psi(\beta^{\tau}) - f(\beta^{\ast}, j({\tau})) + \Psi(\beta^{\ast})\right)$ be the \textit{online regret}\footnote{The online regret is different from the regret in our paper} \cite{xiao_dual}, where $\beta^{\tau}$ be the multiplier of DA at $\tau$. 
Let $S_t$ be $-R_t(\beta) + C\log T$ where $C = \frac{\|v\|^2}{\sigma^2}$, 
let $\Psi^*$ be convex conjugate of $\Psi$ and let $\sigma$ be
$\sigma=\frac{l^2\min_i\in N B_i}{1+\delta_0}.$
\begin{lemma}
\label{lem:distance_upper_bound}
Let $\beta^{\tau} \in \Real^n$ be the solution of DA and $\beta^{*} \in \Real^n$ be the solution of the corresponding EG program.
Then, the following inequality holds for any $t\geq 1$ and $(j({\tau}))_{\tau=1}^t$:
\begin{align*}
    &\frac{\sigma t}{2}\|\beta^{t+1} - \beta^{\ast}\|^2 \leq  t\Psi^{\ast}\left(-\frac{1}{t}s^t\right) + \sum_{\tau=1}^t\left(\langle g^{\tau}, \beta^{\tau} \rangle + \Psi(\beta^{\tau})\right) + \sum_{\tau=1}^t\left(f(\beta^{\ast}, j({\tau})) + \Psi(\beta^{\ast}) - f(\beta^{\tau}, j({\tau})) - \Psi(\beta^{\tau})\right),
\end{align*}
where $g^t\in \partial f(\beta^t, j({\tau})^t)$.
\end{lemma}

\begin{proof}[Proof of Lemma \ref{lem:distance_upper_bound}]
Let us define $s^t = \sum_{\tau=1}^t g^{\tau}$.
From the first-order optimality condition for $\beta^{t+1}$, there exist a subgradient $b^{t+1}\in \partial \Psi(\beta^{t+1})$ such that:
\begin{align}
\label{eq:betat_first_order_opt_cond}
    \langle s^t + t b^{t+1}, \beta - \beta^{t+1}\rangle \geq 0, ~\forall \beta \in \mathrm{dom}(\Psi).
\end{align}
From the strong convexity of $\Psi$, we have:
\begin{align}
\label{eq:psi_strong_convexity}
    \Psi(\beta) \geq \Psi(\beta^{t+1}) + \langle b^{t+1}, \beta - \beta^{t+1}\rangle + \frac{\sigma}{2}\|\beta^{t+1} - \beta\|^2, ~\forall \beta \in \mathrm{dom}(\Psi).
\end{align}
By combining \eqref{eq:betat_first_order_opt_cond} and \eqref{eq:psi_strong_convexity}, we get:
\begin{align*}
    \frac{\sigma t}{2}\|\beta^{t+1} - \beta^{\ast}\|^2 &\leq \langle s^t, \beta^{\ast} - \beta^{t+1}\rangle + t\Psi(\beta^{\ast}) - t\Psi(\beta^{t+1}) \\
    &= \langle -s^t, \beta^{t+1}\rangle + t\Psi(\beta^{\ast}) - t\Psi(\beta^{t+1}) + \langle s^t, \beta^{\ast}\rangle \\
    &= \langle -s^t, \beta^{t+1}\rangle + t\Psi(\beta^{\ast}) - t\Psi(\beta^{t+1}) + \sum_{\tau=1}^t\langle g^{\tau}, \beta^{\ast}\rangle \\
    &= \langle -s^t, \beta^{t+1} \rangle + t\Psi(\beta^{\ast}) - t\Psi(\beta^{t+1}) + \sum_{\tau=1}^t\langle g^{\tau}, \beta^{\ast} - \beta^{\tau}\rangle + \sum_{\tau=1}^t\langle g^{\tau}, \beta^{\tau} \rangle.
\end{align*}
Then, from the convexity of $f(\cdot, j({\tau}))$,
\begin{align*}
    &\frac{\sigma t}{2}\|\beta^{t+1} - \beta^{\ast}\|^2 \\
    &\leq \langle -s^t, \beta^{t+1} \rangle + t\Psi(\beta^{\ast}) - t\Psi(\beta^{t+1}) + \sum_{\tau=1}^t\left(f(\beta^{\ast}, j({\tau})) - f(\beta^{\tau}, j({\tau}))\right) + \sum_{\tau=1}^t\langle g^{\tau}, \beta^{\tau} \rangle \\
    &= \langle -s^t, \beta^{t+1} \rangle - t\Psi(\beta^{t+1}) + \sum_{\tau=1}^t\left(\langle g^{\tau}, \beta^{\tau} \rangle + \Psi(\beta^{\tau})\right) + \sum_{\tau=1}^t\left(f(\beta^{\ast}, j({\tau})) + \Psi(\beta^{\ast}) - f(\beta^{\tau}, j({\tau})) - \Psi(\beta^{\tau})\right).
\end{align*}
Since $\beta^{t+1} = \argmax_{\beta} \{\langle -s^t, \beta \rangle - t\Psi(\beta)\}$, we have:
\begin{align*}
    &\frac{\sigma t}{2}\|\beta^{t+1} - \beta^{\ast}\|^2 \leq  t\Psi^{\ast}\left(-\frac{1}{t}s^t\right) + \sum_{\tau=1}^t\left(\langle g^{\tau}, \beta^{\tau} \rangle + \Psi(\beta^{\tau})\right) + \sum_{\tau=1}^t\left(f(\beta^{\ast}, j({\tau})) + \Psi(\beta^{\ast}) - f(\beta^{\tau}, j({\tau})) - \Psi(\beta^{\tau})\right).
\end{align*}
\end{proof} 

\begin{lemma}
\label{lem:telescoping_sum}
We have for any $t\geq 1$ and $(z_{\tau})_{\tau=1}^t$:
\begin{align*}
    \sum_{\tau=1}^t \left(\langle \beta^{\tau}, g^{\tau}\rangle + \Psi(\beta^{\tau })\right) &\leq - t\Psi^{\ast}\left(-\frac{1}{t}s^t\right) + \Psi(\beta^2) - \Psi(\beta^1) + \frac{1}{2\sigma}\left(\|g^1\|^2 + \sum_{\tau=2}^t\frac{\|g^{\tau}\|^2}{\tau - 1}\right).
\end{align*}
\end{lemma}
\begin{proof}[Proof of Lemma \ref{lem:telescoping_sum}]
We introduce the following lemma:
\begin{lemma}
\label{lem:gradient_of_conjugate}
Assume that a function $v$ is closed and convex and its convex conjugate $v^{\ast}$ is differentiable.
Then, the gradient of $v^{\ast}$ is given by:
\begin{align*}
    \nabla v^{\ast}(y) = \argmax_x \{\langle y, x\rangle - v(x)\}.
\end{align*}
\end{lemma}

From Lemma \ref{lem:gradient_of_conjugate} and the fact that $\beta^{\tau} = \argmax_{\beta}\left\{\left\langle -\frac{1}{\tau - 1}s^{\tau - 1}, \beta \right\rangle - \Psi(\beta)\right\}$, we have:
\begin{align}
\label{eq:v_gradient}
    \nabla \Psi^{\ast}\left(-\frac{1}{\tau-1}s^{\tau-1}\right) &= \beta^{\tau}.
\end{align}
Furthermore, since $\Psi$ is $\sigma$-strongly convex, its convex conjugate is $\frac{1}{\sigma}$-smooth, we get:
\begin{align}
\label{eq:v_smoothness}
    \frac{1}{\tau - 1}\left\langle \nabla \Psi^{\ast}\left(-\frac{1}{\tau-1}s^{\tau-1}\right), g^{\tau}\right\rangle &= \left\langle \nabla \Psi^{\ast}\left(-\frac{1}{\tau-1}s^{\tau-1}\right), -\frac{1}{\tau-1}s^{\tau-1} + \frac{1}{\tau - 1}s^{\tau}\right\rangle \nonumber\\
    &\leq \Psi^{\ast}\left(-\frac{1}{\tau-1}s^{\tau-1}\right) - \Psi^{\ast}\left(-\frac{1}{\tau-1}s^{\tau}\right) + \frac{1}{2\sigma}\left\|-\frac{1}{\tau-1}s^{\tau-1} + \frac{1}{\tau - 1}s^{\tau}\right\|^2 \nonumber\\
    &= \Psi^{\ast}\left(-\frac{1}{\tau-1}s^{\tau-1}\right) - \Psi^{\ast}\left(-\frac{1}{\tau-1}s^{\tau}\right) + \frac{\|g^{\tau}\|^2}{2\sigma(\tau - 1)^2}.
\end{align}
By combining \eqref{eq:v_gradient} and \eqref{eq:v_smoothness}, we have:
\begin{align*}
    \langle \beta^{\tau}, g^{\tau}\rangle &= (\tau - 1)\Psi^{\ast}\left(-\frac{1}{\tau-1}s^{\tau-1}\right) - (\tau - 1)\Psi^{\ast}\left(-\frac{1}{\tau-1}s^{\tau}\right) + \frac{\|g^{\tau}\|^2}{2\sigma(\tau - 1)} \\
    &= (\tau - 1)\Psi^{\ast}\left(-\frac{1}{\tau-1}s^{\tau-1}\right) - (\tau - 1)\max_{\beta}\left\{\left\langle -\frac{1}{\tau - 1}s^{\tau}, \beta \right\rangle - \Psi(\beta)\right\} + \frac{\|g^{\tau}\|^2}{2\sigma(\tau - 1)} \\
    &= (\tau - 1)\Psi^{\ast}\left(-\frac{1}{\tau-1}s^{\tau-1}\right) - \max_{\beta}\left\{\langle -s^{\tau}, \beta \rangle - (\tau - 1)\Psi(\beta)\right\} + \frac{\|g^{\tau}\|^2}{2\sigma(\tau - 1)} \\
    &\leq (\tau - 1)\Psi^{\ast}\left(-\frac{1}{\tau-1}s^{\tau-1}\right) + \langle s^{\tau}, \beta^{\tau + 1} \rangle + (\tau - 1) \Psi(\beta^{\tau + 1}) + \frac{\|g^{\tau}\|^2}{2\sigma(\tau - 1)} \\
    &= (\tau - 1)\Psi^{\ast}\left(-\frac{1}{\tau-1}s^{\tau-1}\right) + \langle s^{\tau}, \beta^{\tau + 1} \rangle + \tau \Psi(\beta^{\tau + 1}) - \Psi(\beta^{\tau + 1}) + \frac{\|g^{\tau}\|^2}{2\sigma(\tau - 1)}.
\end{align*}
Since $\beta^{\tau + 1} = \argmin_{\beta}\left\{\langle s^{\tau}, \beta \rangle + \tau\Psi(\beta)\right\}$,
\begin{align*}
    \langle \beta^{\tau}, g^{\tau}\rangle &\leq (\tau - 1)\Psi^{\ast}\left(-\frac{1}{\tau-1}s^{\tau-1}\right) + \min_{\beta}\left\{\langle s^{\tau}, \beta \rangle + \tau \Psi(\beta)\right\} - \Psi(\beta^{\tau + 1}) + \frac{\|g^{\tau}\|^2}{2\sigma(\tau - 1)} \\
    &= (\tau - 1)\Psi^{\ast}\left(-\frac{1}{\tau-1}s^{\tau-1}\right) - \max_{\beta}\left\{\langle -s^{\tau}, \beta \rangle - \tau \Psi(\beta)\right\} - \Psi(\beta^{\tau + 1}) + \frac{\|g^{\tau}\|^2}{2\sigma(\tau - 1)} \\
    &= (\tau - 1)\Psi^{\ast}\left(-\frac{1}{\tau-1}s^{\tau-1}\right) - \tau\Psi^{\ast}\left(-\frac{1}{\tau}s^{\tau}\right) - \Psi(\beta^{\tau + 1}) + \frac{\|g^{\tau}\|^2}{2\sigma(\tau - 1)}.
\end{align*}
Thus, for any $\tau \geq 2$:
\begin{align}
\label{eq:inner_prod_bound_tau}
    \langle \beta^{\tau}, g^{\tau}\rangle + \Psi(\beta^{\tau + 1}) \leq (\tau - 1) \Psi^{\ast}\left(-\frac{1}{\tau-1}s^{\tau-1}\right) - \tau \Psi^{\ast}\left(-\frac{1}{\tau}s^{\tau}\right) + \frac{\|g^{\tau}\|^2}{2\sigma(\tau - 1)} 
\end{align}

Similarly, from Lemma \ref{lem:gradient_of_conjugate} and the assumption that $\beta^1= \argmin_{\beta}\Psi(\beta)$:
\begin{align}
\label{eq:v_gradient_0}
    \nabla \Psi^{\ast}(0) &= \beta^1.
\end{align}
Furthermore, from the smoothness of $\Psi^{\ast}$, we get:
\begin{align}
\label{eq:v_smoothness_0}
    \langle \nabla \Psi^{\ast}(0), g^1\rangle &= \langle \nabla \Psi^{\ast}(-s^0), -s^0 + s^1\rangle \nonumber\\
    &\leq \Psi^{\ast}(-s^0) - \Psi^{\ast}(-s^1) + \frac{1}{2\sigma}\|-s^0 + s^1\|^2 \nonumber\\
    &= \Psi^{\ast}(-s^0) - \Psi^{\ast}(-s^1) + \frac{\|g^1\|^2}{2\sigma}.
\end{align}
By combining \eqref{eq:v_gradient_0} and \eqref{eq:v_smoothness_0}, we have:
\begin{align*}
    \langle \beta^1, g^1\rangle &\leq \Psi^{\ast}(-s^0) - \Psi^{\ast}(-s^1) + \frac{\|g^1\|^2}{2\sigma}.
\end{align*}
Thus, we have for $\tau = 1$:
\begin{align}
\label{eq:inner_prod_bound_1}
    \langle \beta^1, g^1\rangle + \Psi(\beta^2) \leq \Psi^{\ast}(-s^0) - \Psi^{\ast}(-s^1) + \frac{\|g^1\|^2}{2\sigma} + \Psi(\beta^2).
\end{align}
By summing \eqref{eq:inner_prod_bound_tau} and \eqref{eq:inner_prod_bound_1} from $\tau=1$ to $t$, we have:
\begin{align*}
    \sum_{\tau=1}^t \left(\langle \beta^{\tau}, g^{\tau}\rangle + \Psi(\beta^{\tau + 1})\right) &\leq \Psi^{\ast}(-s^0) - t\Psi^{\ast}\left(-\frac{1}{t}s^t\right) + \Psi(\beta^2) + \frac{1}{2\sigma}\left(\|g^1\|^2 + \sum_{\tau=2}^t\frac{\|g^{\tau}\|^2}{\tau - 1}\right) \\
    &= - t\Psi^{\ast}\left(-\frac{1}{t}s^t\right) + \Psi(\beta^2) - \Psi(\beta^1) + \frac{1}{2\sigma}\left(\|g^1\|^2 + \sum_{\tau=2}^t\frac{\|g^{\tau}\|^2}{\tau - 1}\right).
\end{align*}
Since $\beta^1 = \argmin_{\beta} \Psi(\beta)$, we have $\Psi(\beta^{t+1}) \geq \Psi(\beta^1)$.
Thus, adding $\Psi(\beta^1) - \Psi(\beta^{t+1}) \leq 0$ to the above inequality yields:
\begin{align*}
    \sum_{\tau=1}^t \left(\langle \beta^{\tau}, g^{\tau}\rangle + \Psi(\beta^{\tau })\right) &\leq - t\Psi^{\ast}\left(-\frac{1}{t}s^t\right) + \Psi(\beta^2) - \Psi(\beta^1) + \frac{1}{2\sigma}\left(\|g^1\|^2 + \sum_{\tau=2}^t\frac{\|g^{\tau}\|^2}{\tau - 1}\right).
\end{align*}
\end{proof} 

\begin{lemma}
\label{lem:psi_sub}
We have:
\begin{align*}
    \Psi(\beta^2) - \Psi(\beta^1)  \leq \frac{2}{\sigma}\|g^1\|^2.
\end{align*}
\end{lemma}

\begin{proof}[Proof of Lemma \ref{lem:psi_sub}]
Since $\beta^2 = \argmin_{\beta}\{\langle s^1, \beta\rangle + \Psi(\beta)\}$, we have:
\begin{align*}
    \langle s^1, \beta^2 \rangle + \Psi(\beta^2) \leq \langle s^1, \beta^1 \rangle + \Psi(\beta^1),
\end{align*}
and then:
\begin{align*}
    \Psi(\beta^2) - \Psi(\beta^1) \leq \langle s^1, \beta^1 - \beta^2 \rangle \leq \|s^1\|\|\beta^1 - \beta^2\| = \|g^1\|\|\beta^1 - \beta^2\|.
\end{align*}
On the other hand, by strong convexity of $\Psi$ and the first-order optimality condition for $\beta^1$, we have:
\begin{align*}
    \Psi(\beta^2) \geq \Psi(\beta^1) + \frac{\sigma}{2}\|\beta^1 - \beta^2\|^2.
\end{align*}
By combining these inequalities, we have:
\begin{align*}
    \frac{\sigma}{2}\|\beta^1 - \beta^2\|^2 \leq \|g^1\|\|\beta^1 - \beta^2\|,
\end{align*}
and then:
\begin{align*}
    \|\beta^1 - \beta^2\| \leq \frac{2}{\sigma}\|g^1\|.
\end{align*}
Therefore, we have:
\begin{align*}
    \Psi(\beta^2) - \Psi(\beta^1)  \leq \frac{2}{\sigma}\|g^1\|^2.
\end{align*}
\end{proof}

\begin{lemma}\label{lem_UCB_proof2}

For any $t\geq 1$ and $(j({\tau}))_{\tau=1}^t$:
\begin{align}\label{ineq_beta_bound5}
\|\beta^{\DA,\UCB(k)}-\beta^{*,\UCB(k)}\|^2\le O(\frac{C\log T_k - R_{T_k}(\beta)}{T_k}).
\end{align}
where $C = \frac{\|v\|^2}{\sigma^2} $ 
\end{lemma}

\begin{proof}[Proof of Lemma \ref{lem_UCB_proof2}]
From lemma \ref{lem:distance_upper_bound} ,for any $t\geq 1$ and $(j({\tau}))_{\tau=1}^t$,
\begin{align}\label{ineq_beta_bound1}
    \frac{\sigma t}{2}\|\beta^{t+1} - \beta^{\ast}\|^2 &\leq  t\Psi^{\ast}\left(-\frac{1}{t}s^t\right) + \sum_{\tau=1}^t\left(\langle v^{\tau}, \beta^{\tau} \rangle + \Psi(\beta^{\tau})\right) + \sum_{\tau=1}^t\left(f(\beta^{\ast}, j({\tau})) + \Psi(\beta^{\ast}) - f(\beta^{\tau}, j({\tau})) - \Psi(\beta^{\tau})\right)\\
    & = t\Psi^{\ast}\left(-\frac{1}{t}s^t\right) + \sum_{\tau=1}^t\left(\langle v^{\tau}, \beta^{\tau} \rangle + \Psi(\beta^{\tau})\right) - R_t(\beta).
\end{align}

$R_t(\beta)=\sum_{\tau=1}^t\left(f(\beta^{\tau}, j({\tau})) - \Psi(\beta^{\tau}) - f(\beta^{\ast}, j({\tau})) + \Psi(\beta^{\ast})\right)$.

By combining this inequality and lemma \ref{lem:telescoping_sum},
\begin{align}\label{ineq_beta_bound2}
    \frac{\sigma t}{2}\|\beta^{t+1} - \beta^{\ast}\|^2 &\leq \Psi(\beta^2) - \Psi(\beta^1) + \frac{1}{2\sigma}\left(\|g^1\|^2 + \sum_{\tau=2}^t\frac{\|g^{\tau}\|^2}{\tau - 1}\right) - R_t(\beta).
\end{align}

By combining this inequality and lemma \ref{lem:psi_sub},
\begin{align}\label{ineq_beta_bound3}
    \frac{\sigma t}{2}\|\beta^{t+1} - \beta^{\ast}\|^2 &\leq \frac{2}{\sigma}\|g^1\|^2+ \frac{1}{2\sigma}\left(\|g^1\|^2 + \sum_{\tau=2}^t\frac{\|g^{\tau}\|^2}{\tau - 1}\right) - R_t(\beta).
\end{align}
Thus,

\begin{align}\label{ineq_beta_bound4}
    \|\beta^{t+1} - \beta^{\ast}\|^2 &\leq \frac{2}{\sigma^2 t}\left(2\|v^1\|^2 + \frac{1}{2}\left(\|v^1\|^2 + \sum_{\tau=2}^t\frac{\|v^{\tau}\|^2}{\tau - 1}\right)\right) - \frac{2}{\sigma t}R_t(\beta).
\end{align}

This inequality implies that, 
\begin{align}
    \|\beta^{t+1} - \beta^{\ast}\|^2 &\leq O(\frac{C\log t - R_t(\beta)}{t}).
\end{align}
where $C = \frac{\|v\|^2}{\sigma^2} $.

\end{proof}

\begin{lemma}
\label{lem:utility_bound}
For any $k$,
\begin{align*}
\|u^{\DA,\RUCB(k)}-u^{*,\RUCB(k)}\|\le O(\frac{1}{\min_i B_i}\sqrt{\frac{C\log T_k - R_{T_k}(\beta)}{T_k}}).
\end{align*}
\end{lemma}
\begin{proof}[Proof of Lemma \ref{lem:utility_bound}]
From lemma \ref{lem_UCB_proof2},
\begin{align}
\|\beta^{\DA,\UCB(k)}-\beta^{*,\UCB(k)}\|^2\le O(\frac{C\log T_k - R_{T_k}(\beta)}{T_k}),
\end{align}
which implies
\begin{align}\label{ineq_beta_bound_UCB}
\|\beta^{\DA,\UCB(k)}-\beta^{*,\UCB(k)}\|\le O(\sqrt{\frac{C\log T_k - R_{T_k}(\beta)}{T_k}}).
\end{align}
By definition of $\beta$, 
\begin{align}
\beta^{*,\UCB(k)}_i = \frac{B_i}{u^{*,\RUCB(k)}_i}\\
\beta^{\DA,\UCB(k)}_i = O\left( \frac{B_i}{u^{\DA,\RUCB(k)}_i} \right). 
\end{align}

By combining (\ref{ineq_beta_bound_UCB}) and definition of $\beta$,
\begin{align}
(\min_i B_i)\|\frac{1}{u^{\DA,\RUCB(k)}}-\frac{1}{u^{*,\RUCB(k)}}\|\le O(\sqrt{\frac{C\log T_k - R_{T_k}(\beta)}{T_k}}).
\end{align}
Thus,
\begin{align}\label{ineq_utility_bound}
(\min_i B_i)\|\frac{u^{\DA,\RUCB(k)}-u^{*,\RUCB(k)}}{u^{\DA,\RUCB(k)}u^{*,\RUCB(k)}}\|\le O(\sqrt{\frac{C\log T_k - R_{T_k}(\beta)}{T_k}}).
\end{align}
By $\|v\|_\infty \le 1$, we get $0\le u^{\DA,\RUCB(k)}_i \le1$, $0\le u^{*,\RUCB(k)}_i \le1$.

Thus, (\ref{ineq_utility_bound}) imply that,
\begin{align}
\|u^{\DA,\RUCB(k)}-u^{*,\RUCB(k)}\|\le O(\frac{1}{\min_i B_i}\sqrt{\frac{C\log T_k - R_{T_k}(\beta)}{T_k}})
\end{align}
Moreover, combining this inequality and $0\le u_i^{\DA,\RUCB(k)}u_i^{*,\RUCB(k)} \le1$,

\begin{align}\label{ineq_utility_bound_UCB}
\|u^{\DA,\RUCB(k)}-u^{*,\RUCB(k)}\|\le O(\frac{1}{\min_i B_i}\sqrt{\frac{C\log T_k - R_{T_k}(\beta)}{T_k}})
\end{align}
\end{proof} 

The following lemma bounds the online regret $R_\tau$ uniformly over the rounds $\tau$.
\begin{lemma}{\rm (Martingale bound for DA)}\label{lem_martingale}
Assume that $\Delta_t(\beta) \le C \log T$ for all $T$. Then, for any $a>0$
\begin{equation}\label{ineq_martingale}
\Pr[
\sup_{\tau} (-R_\tau) + C\log T \ge  a
] \le \frac{C \log T}{a}.
\end{equation}
\end{lemma}
\begin{proof}[Proof of Lemma \ref{lem_martingale}]
Since $R_t \le C \log T$ (Theorem 1 (a) in Xiao) and $R_t$ is a submartingale, $S_t = -R_t + C \log T$ is a non-negative supermartingale. Ville's inequality implies for any supermartingale 
\begin{equation}
\Pr[\sup_\tau S_\tau \ge a] \le \frac{S_0}{a},
\end{equation}
which is Eq.~\eqref{ineq_martingale}.
\end{proof} 

\begin{lemma}
\label{lem:utility_bound_martingale}
Let 
\begin{equation}\label{ineq:utility_bound_martingale}
\mX_{ka} = \left\{
\|u^{\DA,\RUCB(k)}-u^{*,\RUCB(k)}\|\le O(\frac{1}{\min_i B_i}\sqrt{\frac{a}{T_k}})
\right\}.
\end{equation}
With probability at least $1 - \frac{C \log T}{a}$, event $\mX_{ka}$ holds.
\end{lemma}
\begin{proof}[Proof of Lemma \ref{lem:utility_bound_martingale}]
The lemma immediately follows from Lemmas \ref{lem:utility_bound} and \ref{lem_martingale}.
\end{proof}

\subsubsection{Auxiliary lemmas over the multiple instances of DA}\label{subsec_rucb_multipleda}

Next, we introduce the following lemma:
\begin{lemma}
\label{lem:utility_bound_UCB}
For any $a > 0$, 
\begin{align*}
\Pr[\bigcup_k \mX_{ka}]  \le \frac{nmC(\log T)^2}{a}.
\end{align*}
\end{lemma}
\begin{proof}[Proof of Lemma \ref{lem:utility_bound_UCB}]
DA-UCB-Reset resets the DA subroutine at most $nm\log_2 T$ times. Thus, from union bound of Lemma \ref{lem:utility_bound_martingale} over all instances of DA yields 
\begin{equation}
\Pr[\bigcup_k \mX_{ka}] \le \sum_k \Pr[\mX_{ka}] \le \frac{nmC(\log T)^2}{a}.
\end{equation}
\end{proof} 

Next, We introduce the following lemma:
\begin{lemma}
\label{lem:utility_bound_UCB2}
For any $a > 0$, with $\bigcap_k \mX^c_{ka}$,
\begin{align*}
\Ind[\bigcap_k \mX^c_{ka}]|u_i^{*,\RUCB} - u_i^{\DA,\RUCB}|
&\le \tilde{O}(\sqrt{\frac{anm}{T(\min_i B_i)^2}}).
\end{align*}
\end{lemma}
\begin{proof}[Proof of Lemma \ref{lem:utility_bound_UCB2}]
By definition, 
\begin{align}
\Ind[\mX_{ka}]\|u^{\DA,\RUCB(k)}-u^{*,\RUCB(k)}\|\le O(\frac{1}{\min_i B_i}\sqrt{\frac{a}{T_k}}),
\end{align}
which implies,
\begin{align}\label{ineq_utility_boynd_UCB2}
\Ind[\bigcap_k \mX^c_{ka}]|u_i^{\DA,\RUCB(k)}-u_i^{*,\RUCB(k)}|\le O(\frac{1}{\min_i B_i}\sqrt{\frac{a}{T_k}}).
\end{align}

By definition of $u_i^{*,\RUCB}$ and $u_i^{\DA,\RUCB}$,
\begin{align}
u_i^{*,\RUCB}&=\frac{1}{T}(T_1u_i^{*,\RUCB(1)}+T_2u_i^{*,\RUCB(2)}+\dots+T_yu_i^{*,\RUCB(y)}),\\
u_i^{\DA,\RUCB}&=\frac{1}{T}\sum_{t=1}^T \Ind[I(t)=i] \vucb_{i(t),j(t),N_{i,j(t)}(s)}\\
&=\frac{1}{T}(T_1u_i^{\DA,\RUCB(1)}+T_2u_i^{\DA,\RUCB(2)}+\dots+T_yu_i^{\DA,\RUCB(y)}).
\end{align}

Thus, from the triangle inequality,
\begin{align}
\Ind[\bigcap_k \mX^c_{ka}]|u_i^{*,\RUCB} - u_i^{\DA,\RUCB}| &\le \Ind[\bigcap_k \mX^c_{ka}](\frac{T_1}{T}|u_i^{*,\RUCB(1)}-u_i^{\DA,\RUCB(1)}| + \frac{T_2}{T}|u_i^{*,\RUCB(2)}-u_i^{\DA,\RUCB(2)}|\\
&+ \dots + \frac{T_y}{T}|u_i^{*,\RUCB(y)}-u_i^{\DA,\RUCB(y)}|).
\end{align}

For all $k=1,2,\dots,y$, from (\ref{ineq_utility_boynd_UCB2}),
\begin{align}
\Ind[\bigcap_k \mX^c_{ka}]\frac{T_k}{T}|u_i^{*,\RUCB(k)}-u_i^{\DA,\RUCB(k)}| &\le O(\frac{T_k}{T\min_i B_i}\sqrt{\frac{a}{T_k}}) \\
&= \tilde{O}(\frac{\sqrt{T_k a}}{T\min_i B_i}).
\end{align}

Thus, with $\bigcap_k \mX^c_{ka}$,
\begin{align}
\Ind[\bigcap_k \mX^c_{ka}]|u_i^{*,\RUCB} - u_i^{\DA,\RUCB}| &\le \frac{T_1}{T}|u_i^{*,\RUCB(1)}-u_i^{\DA,\RUCB(1)}| + \frac{T_2}{T}|u_i^{*,\RUCB(2)}-u_i^{\DA,\RUCB(2)}|\\
&+ \dots + \frac{T_k}{T}|u_i^{*,\RUCB(k)}-u_i^{\DA,\RUCB(k)}|\\
&\le \tilde{O}(\frac{a^{\frac{1}{2}}(\sqrt{T_1}+\sqrt{T_2}+\dots +\sqrt{T_k})}{T\min_i B_i}).
\end{align}

Here, DA-UCB-Reset resets the DA subroutine at most $|NM|\log_2 T$ times, and from Cauchy–Schwarz inequality, 
\begin{align}
&\sqrt{T_1}+\sqrt{T_2}+\dots +\sqrt{T_k}\le \sqrt{(nm\log_2 T) T}\\
\end{align}
Therefore,
\begin{align}
\Ind[\bigcap_k \mX^c_{ka}]|u_i^{*,\RUCB} - u_i^{\DA,\RUCB}|
&\le \tilde{O}(\frac{\sqrt{a}(\sqrt{T_1}+\sqrt{T_2}+\dots +\sqrt{T_k})}{T\min_i B_i})\\
&\le \tilde{O}(\frac{\sqrt{anm T}}{T\min_i B_i})\\
&\le \tilde{O}(\sqrt{\frac{anm}{T(\min_i B_i)^2}}).
\end{align}
\end{proof}

\subsubsection{Bound on the second term}
\label{subsec_termtwo_mainlemmas}

\begin{lemma}\label{lem_termtwo_tail}
Assume that $B_i = 1/n$.
Let $\Termtwo = T |\prod_i u_i^{*,\RUCB} - \prod_i u_i^{\DA,\RUCB}|$. Then, for any $x>0$, we have
\begin{equation}
\Pr[\Termtwo \ge x] \le \min\left(1, \frac{n^4m^2C^3(\log T)^2 T}{x^2}\right).
\end{equation}
\end{lemma}
\begin{proof}[Proof of Lemma \ref{lem_termtwo_tail}]
Lemmas \ref{lem:utility_bound_UCB} and \ref{lem:utility_bound_UCB2} with $min_i B_i = 1/n$ imply there exists a constant $C>0$ such that
\begin{equation}\label{ineq_termtwo_original}
\Pr\left[X \ge C \sqrt{an^3m T}\right] 
\le \frac{nmC(\log T)^2}{a}.
\end{equation}
Substituting $a = x^2(n^3mT)^{-1}C^{-2}$, we can see that Eq.~\eqref{ineq_termtwo_original} is equivalent to
\begin{equation}\label{ineq_termtwo_modified}
\Pr\left[X \ge x\right] 
\le \frac{n^4m^2C^3(\log T)^2 T}{x^2}.
\end{equation}
Moreover, $\Pr[\Termtwo \ge x] \le 1$ by definition of probability. Combining this with Eq.~\eqref{ineq_termtwo_modified} completes the lemma.
\end{proof}

\begin{lemma}{\rm (survival function)}\label{lem_survival}

Let $X$ be non-negative random variable such that
\begin{equation}\label{ineq_termtwotail}
P[X > x] \le C\min\left(
\frac{T}{x^2}, 1
\right)
\end{equation}
for some $C>0$. Then, 
\begin{equation}
\Ex[X] \le 2 C \sqrt{T}.
\end{equation}
\end{lemma}
\begin{proof}[Proof of Lemma \ref{lem_survival}]    

Let $S^{ub}(x)$ be such that $\Pr[X > x] \le S^{ub}(x)$. Then,
\begin{align}
E[X]
&= \int_0^\infty \Pr[X > x] dx \text{\ \ \ \ (by def of survival function)}\\
&\le \int_0^\infty S^{ub}(x) dx  \\
&\le C \int_0^\infty \min(T/x^2, 1) dx\\
&\le C \int_{\sqrt{T}}^\infty T/x^2 dx + \int_0^{\sqrt{T}} dx\\
&\le C \left(
[-T/x]_{\sqrt{T}}^\infty dx + \sqrt{T} 
\right) \\
&\le C \left( \sqrt{T} + \sqrt{T} \right),
\end{align}
which completes the proof of Lemma \ref{lem_survival}.
\end{proof} 

\begin{lemma}\label{lemma_secondterm_newmain}
We have
\begin{equation}\label{ineq_secondterm_newmain}
\Ex\left[T\left|\prod_i (u_i^{*,\RUCB})^{B_i}) - \prod_i (u_i^{\DA,\RUCB})^{B_i}\right| \right] = 
\tilde{O}\left( n^4 m^2 \sqrt{T} \right).
\end{equation}
\end{lemma}
\begin{proof}[Proof of Lemma \ref{lemma_secondterm_newmain}]
Combining Lemmas \ref{lem_termtwo_tail} and \ref{lem_survival} with $C = n^4m^2C^3(\log T)^2$ yields Eq.~\eqref{ineq_secondterm_newmain}.
\end{proof} 

\subsection{Lemmas on the third term}

\begin{lemma}\label{lem_UCB_proof3}
for any agent i $\in N$, with $\mG_1 \cap \mG_2$ from lemma \ref{lem_ucbreset_rewards},
\begin{align}
\Ind[\mG_1 \cap \mG_2]
|(\prod_i (u_i^{\DA,\RUCB})^{B_i}) - \prod_i (u_i^{\DA,\true})^{B_i})|
= 
\tilde{O}\left(
\frac{nm\sqrt{T}}{T}
\right).
\end{align}
\end{lemma}

\begin{proof}[Proof of Lemma \ref{lem_UCB_proof3}]
The proof is immediately derived from the definition of $\mG_2$.

\end{proof}

\section{Proofs on DA}
Let $v_i = (v_{i,1},v_{i,2},\dots,v_{i,m})$ and $s = (s_1,s_2,\dots,s_m)$ be corresponding vectors of size $m$ and $\langle v_i, s\rangle$ be an inner product of them.
Let $\mathbf{e}^{i}$ be the unit vector of the $i$-th coordinate. 
Let $\beta = (\beta_1,\beta_2,\dots,\beta_n)$ and its value at the beginning of $t$ be $\beta^t$.
Let $B = (B_1, B_2, \dots, B_n)$ and $\mathbf{1} = (1,1,\dots,1)$.

\subsection{Proof of Lemma \ref{lem_dabound}}
\begin{proof}[Proof of Lemma \ref{lem_dabound}]
Let us consider the following event that no projection occurs when updating $\beta_i^{t+1}$ for buyer $i\in N$ and $t\geq 1$:
\begin{align*}
    \mV_i^t := \left\{\frac{B_i}{\bar{u}_i^{\DA}(t)}\in \left[\frac{B_i}{h(1+\delta_0)}, \frac{1+\delta_0}{l}\right]\right\}.
\end{align*}
From Lemma \ref{lem:beta_ast_bound}, whenever the complementary event $(\mV_i^t)^c = \left\{\frac{B_i}{\bar{u}_i^{\DA}(t)}\notin \left[\frac{B_i}{h(1+\delta_0)}, \frac{1+\delta_0}{l}\right]\right\}$ occurs, it holds that:
\begin{align*}
    |\beta_i^{t+1} - \beta_i^{\ast, \DA}| > \min\left(\frac{1+\delta_0}{l} - \beta_i^{\ast, \DA}, \beta_i^{\ast, \DA} - \frac{B_i}{h(1+\delta_0)}\right) > 0.
\end{align*}
Therefore, we have:
\begin{align*}
    \mathbb{E}\left[(\beta_i^{t+1} - \beta_i^{\ast, \DA})^2\right] &= \mathbb{E}\left[(\beta_i^{t+1} - \beta_i^{\ast, \DA})^2 ~|~ \mV_i^t\right]\mathbb{P}(\mV_i^t) + \mathbb{E}\left[(\beta_i^{t+1} - \beta_i^{\ast, \DA})^2 ~|~ (\mV_i^t)^c\right]\mathbb{P}((\mV_i^t)^c) \\
    &\geq \mathbb{E}\left[(\beta_i^{t+1} - \beta_i^{\ast, \DA})^2 ~|~ (\mV_i^t)^c\right]\mathbb{P}((\mV_i^t)^c) \\
    &\geq \left(\min\left(\frac{1+\delta_0}{l} - \beta_i^{\ast, \DA}, \beta_i^{\ast, \DA} - \frac{B_i}{h(1+\delta_0)}\right)\right)^2\mathbb{P}((\mV_i^t)^c).
\end{align*}
Then, we get:
\begin{align}
    \label{eq:conditioning_complemenraty_event}
    \mathbb{P}((\mV_i^t)^c) \leq \frac{1}{\left(\min\left(\frac{1+\delta_0}{l} - \beta_i^{\ast, \DA}, \beta_i^{\ast, \DA} - \frac{B_i}{h(1+\delta_0)}\right)\right)^2}\mathbb{E}\left[(\beta_i^{t+1} - \beta_i^{\ast, \DA})^2\right].
\end{align}

On the other hand, conditioning on $\mV_i^t$, we have
\begin{align}
    \label{eq:conditioning_event}
    \bar{u}_i^{\DA}(t) = \frac{B_i}{\beta_i^{t+1}}.
\end{align}

Moreover, from Lemma \ref{lem:beta_ast_bound} and the assumption that $l\leq \langle v_i^{\DA}, s\rangle$, we get:
\begin{align}
    \label{eq:u_ast_ub_by_vnorm}
    u_i^{\ast, \DA} \leq h = \frac{h}{l}l \leq \frac{h}{l}\langle v_i^{\DA},s\rangle \leq \frac{h}{l}\|v_i^{\DA}\|_{\infty}\|s\|_1 = \frac{h}{l}\|v_i^{\DA}\|_{\infty},
\end{align}
where the last inequality follows from H\"{o}lder's inequality.

By combining \eqref{eq:conditioning_complemenraty_event}, \eqref{eq:conditioning_event}, and \eqref{eq:u_ast_ub_by_vnorm}, we have:
\begin{align*}
    &\mathbb{E}\left[(\bar{u}_i^{\DA}(T) - u_i^{\ast, \DA})^2\right] \\
    &= \mathbb{E}\left[\mathbb{I}\{(\mV_i^T)^c\}(\bar{u}_i^{\DA}(T) - u_i^{\ast, \DA})^2\right] + \mathbb{E}\left[\mathbb{I}\{\mV_i^T\}(\bar{u}_i^{\DA}(T) - u_i^{\ast, \DA})^2\right] \\
    &= \mathbb{E}\left[\mathbb{I}\{(\mV_i^T)^c\}(\bar{u}_i^{\DA}(T) - u_i^{\ast, \DA})^2\right] + \mathbb{E}\left[\mathbb{I}\{\mV_i^T\}\left(\frac{B_i}{\beta_i^{T+1}} - u_i^{\ast, \DA}\right)^2\right] \\
    &\leq \mathbb{E}\left[\mathbb{I}\{(\mV_i^T)^c\}\left(\max(u_i^{\ast, \DA}, \|v_i^{\DA}\|_{\infty})\right)^2\right] + \mathbb{E}\left[\mathbb{I}\{\mV_i^T\}\left(\frac{B_i}{\beta_i^{T+1}} - u_i^{\ast, \DA}\right)^2\right] \\
    &\leq \frac{h}{l}\|v_i^{\DA}\|_{\infty}^2\mathbb{E}\left[\mathbb{I}\{(\mV_i^T)^c\}\right] + (u_i^{\ast, \DA})^2\mathbb{E}\left[\mathbb{I}\{\mV_i^T\}\left(\frac{B_i}{\beta_i^{T+1}u_i^{\ast, \DA}} - 1\right)^2\right] \\
    &= \frac{h}{l}\|v_i^{\DA}\|_{\infty}^2\mathbb{P}\left[(\mV_i^T)^c\right] + (u_i^{\ast, \DA})^2\mathbb{E}\left[\mathbb{I}\{\mV_i^T\}\left(\frac{\beta_i^{\ast, \DA}}{\beta_i^{T+1}} - 1\right)^2\right] \\
    &\leq \frac{h\|v_i^{\DA}\|_{\infty}^2}{l\left(\min\left(\frac{1+\delta_0}{l} - \beta_i^{\ast, \DA}, \beta_i^{\ast, \DA} - \frac{B_i}{h(1+\delta_0)}\right)\right)^2}\mathbb{E}\left[(\beta_i^{T+1} - \beta_i^{\ast, \DA})^2\right] + \left(\frac{h(1+\delta_0)u_i^{\ast, \DA}}{B_i}\right)^2\mathbb{E}\left[(\beta_i^{T+1} - \beta_i^{\ast, \DA})^2 \right],
\end{align*}
where the first inequality follow from $0\leq \bar{u}_i^{DA}\leq \|v_i^{\DA}\|_{\infty}$, and the third equality follows from $\beta_i^{\ast,\DA}=\frac{B_i}{u_i^{\ast, \DA}}$ by Theorem 1 in \cite{gao:pace:2021}.
Since $\frac{B_i}{h}\leq \beta_i^{\ast, \DA}\leq \frac{1}{l}$ from Lemma \ref{lem:beta_ast_bound}, 
\begin{align*}
    \mathbb{E}\left[(\bar{u}_i^{\DA}(T) - u_i^{\ast, \DA})^2\right] &\leq \left(\frac{h\|v_i^{\DA}\|_{\infty}^2}{l\left(\min\left(\frac{\delta_0}{l}, \frac{B_i\delta_0}{h(1+\delta_0)}\right)\right)^2} + \left(\frac{h^2(1+\delta_0)}{B_i}\right)^2\right)\mathbb{E}\left[(\beta_i^{T+1} - \beta_i^{\ast, \DA})^2\right] \\
    &\leq \left(\left(\frac{(1+\delta_0)}{\delta_0}\right)^2\frac{h^3\|v_i^{\DA}\|_{\infty}^2}{lB_i^2} + \left(\frac{h^2(1+\delta_0)}{B_i}\right)^2\right)\mathbb{E}\left[(\beta_i^{T+1} - \beta_i^{\ast, \DA})^2\right] \\
    &\leq \left(\left(\frac{(1+\delta_0)}{\delta_0}\right)^2\frac{h^3\|v_i^{\DA}\|_{\infty}^2}{l\left(\min_{i\in N}B_i\right)^2} + \left(\frac{h^2(1+\delta_0)}{\min_{i\in N}B_i}\right)^2\right)\mathbb{E}\left[(\beta_i^{T+1} - \beta_i^{\ast, \DA})^2\right],
\end{align*}
where the second inequality follows from $\frac{\delta_0}{l} \geq \frac{\delta_0}{h} \geq \frac{\delta_0}{h(1+\delta_0)} \geq \frac{B_i\delta_0}{h(1+\delta_0)}$.
From Lemma \ref{lem:convergence_beta}, by summing up the above inequality for $i\in N$, we have:
\begin{align*}
    \mathbb{E}\left[\|\bar{u}^{\DA}(T) - u^{\ast, \DA}\|^2\right] &\leq \sum_{i\in N}\left(\left(\frac{(1+\delta_0)}{\delta_0}\right)^2\frac{h^3\|v_i^{\DA}\|_{\infty}^2}{l\left(\min_{i\in N}B_i\right)^2} + \left(\frac{h^2(1+\delta_0)}{\min_{i\in N}B_i}\right)^2\right)\mathbb{E}\left[(\beta_i^{T+1} - \beta_i^{\ast, \DA})^2\right] \\
    &\leq \left(\left(\frac{(1+\delta_0)}{\delta_0}\right)^2\frac{h^3\|v^{\DA}\|_{\infty}^2}{l\left(\min_{i\in N}B_i\right)^2} + \left(\frac{h^2(1+\delta_0)}{\min_{i\in N}B_i}\right)^2\right)\mathbb{E}\left[\|\beta^{T+1} - \beta^{\ast, \DA}\|^2\right] \\
    &\leq \left(\left(\frac{(1+\delta_0)}{\delta_0}\right)^2\frac{h^3\|v^{\DA}\|_{\infty}^2}{l\left(\min_{i\in N}B_i\right)^2} + \left(\frac{h^2(1+\delta_0)}{\min_{i\in N}B_i}\right)^2\right)\frac{G^2}{\mu^2 T}(6 + \log T),
\end{align*}
where $G=\|v^{\DA}\|_{\infty}$ and $\mu=\frac{l^2\min_{i\in N}B_i}{(1+\delta_0)^2}$.
\end{proof}


\section{Additional Lemmas on DA}
\begin{lemma}
\label{lem:beta_ast_bound}
Assume that the values of items are given by the deterministic values $\{v^{\DA}_{i,j}\}_{i\in N, j\in M}$.
Furthermore, assume that the parameters $l,h$ in DA satisfies that $l\leq \sum_{j\in M}s_jv_{i,j}^{\DA}\leq h$ for all $i\in N$.
Then, the equilibrium utilities satisfy $B_il\leq u_i^{\ast,\DA}\leq h$ and hence $\frac{B_i}{h}\leq \beta_i^{\ast, \DA} \leq \frac{1}{l}$.
\end{lemma}

\begin{lemma}
\label{lem:psi_strongly_convexity}
$\Psi(\beta):= -\sum_{i\in N}B_i\log \beta_i$ is $\mu$-strongly convex on $[\frac{B}{h(1+\delta_0)}, \frac{1+\delta_0}{l}\mathbf{1}]$ with $\mu=\frac{l^2\min_{i\in N}B_i}{(1+\delta_0)^2}$.
\end{lemma}

\begin{lemma}
\label{lem:convergence_beta}
Assume that the values of items $\{v^{\DA}_{i,j(t)}\}$ are deterministic.
Furthermore, assume that $l,h>0$ satisfies that $l\leq \sum_{j\in M}s_jv_{i,j}^{\DA}\leq h$ for all $i\in N$, and $\beta^1 = \argmin_{\left[\frac{B}{h(1+\delta_0)}, \frac{1+\delta_0}{l}\mathbf{1}\right]} \Psi(\beta)$.
Let $\beta^{*,\DA}$ be the optimal solution of \eqref{eq:optimization_problem}.
Then, the following inequality holds for all $t\geq 1$:
\begin{align*}
    &\mathbb{E}\left[\|\beta^{t+1} - \beta^{*,\DA}\|^2\right] \leq \frac{G^2}{\mu^2 t}(6 + \log t),
\end{align*}
where $G=\|v^{\DA}\|_{\infty}:=\max_{i\in N}\|v_i^{\DA}\|_{\infty}$ and $\mu=\frac{l^2\min_{i\in N}B_i}{(1+\delta_0)^2}$.
\end{lemma}

\subsection{Proof of Lemma \ref{lem:beta_ast_bound}}
\begin{proof}[Proof of Lemma \ref{lem:beta_ast_bound}]
For buyer $i\in N$, the largest utility is attainable when the entire set of items is allocated to $i$ (given by the supply $s$).
Thus, from the assumption that $\sum_{j\in M}s_jv_{i,j}^{\DA} \leq h$, we have:
\begin{align}
    \label{eq:u_ast_ub}
    u_i^{\ast,\DA} \leq \langle v_i^{\DA}, s\rangle \leq h.
\end{align}

On the other hand, from Theorem 1 in \cite{gao:pace:2021}, we have for any market equilibrium $(x^{\ast,\DA}, p^{\ast,\DA})$:
\begin{align*}
    \langle p^{\ast,\DA}, x_i^{\ast,\DA}\rangle = \beta_i^{\ast,\DA}\langle v_i, x_i^{\ast,\DA}\rangle = B_i.
\end{align*}
By the assumption that $\sum_{i\in N}B_i=1$ and market clearance $\langle p^{\ast,\DA}, s - \sum_{i\in N}x_i^{\ast,\DA}\rangle = 0$, we get:
\begin{align*}
    \langle p^{\ast,\DA}, s\rangle = \sum_{i\in N}\langle p^{\ast,\DA}, x_i^{\ast,\DA}\rangle = \sum_{i\in N}B_i = 1.
\end{align*}
Thus,
\begin{align*}
    \langle p^{\ast,\DA}, B_is\rangle = B_i.
\end{align*}
This means that each buyer $i$ can afford the proportional allocation $x_i^{\circ}:=B_is$ under the item price $p^{\ast,\DA}$.
Therefore, from the buyer optimality of the market equilibrium and the assumption that $l\leq \sum_{j\in M}s_jv_{i,j}^{\DA}$, we have:
\begin{align}
    \label{eq:u_ast_lb}
    u_i^{\ast,\DA} \geq \langle v_i^{\DA}, x_i^{\circ}\rangle = B_i\langle v_i^{\DA}, s\rangle \geq B_il.
\end{align}

By combining \eqref{eq:u_ast_ub} and \eqref{eq:u_ast_lb}, we get:
\begin{align*}
    B_il\leq u_i^{\ast,\DA} \leq h.
\end{align*}
Moreover, since $\beta_i^{\ast,\DA}=\frac{B_i}{u_i^{\ast,\DA}}$ from Theorem 1 in \cite{gao:pace:2021}, we have:
\begin{align*}
    \frac{B_i}{h} \leq \beta_i^{\ast,\DA}  \leq \frac{1}{l}.
\end{align*}
\end{proof}

\subsection{Proof of Lemma \ref{lem:psi_strongly_convexity}}
\begin{proof}[Proof of Lemma \ref{lem:psi_strongly_convexity}]
The Hessian matrix of $\Psi$ at $\beta\in \left[\frac{B}{h(1+\delta_0)}, \frac{1+\delta_0}{l}\mathbf{1}\right]$ is given by:
\begin{align*}
\nabla^2 \Psi(\beta) = 
\begin{pmatrix}
\frac{B_1}{\beta_1^2} & &\\
& \ddots & \\
& & \frac{B_n}{\beta_n^2}
\end{pmatrix}.
\end{align*}
Thus, the minimum eigenvalue of $\nabla^2 \Psi(\beta)$ is lower bounded as:
\begin{align*}
    \lambda_{\min}(\nabla^2 \Psi(\beta)) \geq \min_{i\in N}\frac{B_i}{\beta_i^2} \geq \frac{l^2\min_{i\in N}B_i}{(1 + \delta_0)^2}.
\end{align*}
Therefore, $\Psi$ is $\mu$-strongly convex on $[\frac{B}{h(1+\delta_0)}, \frac{1+\delta_0}{l}\mathbf{1}]$ with $\mu=\frac{l^2\min_{i\in N}B_i}{(1+\delta_0)^2}$.
\end{proof}

\subsection{Proof Lemma \ref{lem:convergence_beta}}
\begin{proof}[Proof of Lemma \ref{lem:convergence_beta}]

Let us assume that $l\leq \langle v_i, s\rangle\leq h$ for $i\in N$.
Then, $D_{EG}$ can written as:
\begin{align}
    \label{eq:optimization_problem}
    \mathop{\rm minimize}\limits_{\beta \in \left[\frac{B}{h(1+\delta_0)}, \frac{1+\delta_0}{l}\mathbf{1}\right]} \left\{\phi(\beta) := \mathbb{E}_{j\sim s}\left[\max_{i\in N}v_{i,j}\beta_i\right] -\sum_{i\in N}B_{i} \mathrm{log}\beta_i\right\}.
\end{align}
Defining $f(\beta, j)=\max_{i\in N} v_{i,j}\beta_i$ and $\Psi(\beta)=-\sum_{i\in N}B_i\log \beta_i$, we have $v_{i^{\ast}_j,j}\mathbf{e}^{(i^{\ast}_j)}\in \partial f(\beta, j)$, where $i^{\ast}_j \in \argmax_{i\in N} v_{i,j}\beta_i$.
The update rule of DA is given as:
\begin{align}
    \beta^{t+1} &= \left(\proj_{[\frac{B_i}{h(1+\delta_0)}, \frac{1+\delta_0}{l}]}\left(\frac{B_i}{\frac{1}{t}\sum_{\tau=1}^t v_{i, j(\tau)}\mathbb{I}\{i=i^{\ast}_{j(\tau)}\}}\right)\right)_{i\in N}, \nonumber\\
    &= \argmin_{\beta \in \left[\frac{B}{h(1+\delta_0)}, \frac{1+\delta_0}{l}\mathbf{1}\right]} \left\{\frac{1}{t}\sum_{\tau=1}^t\langle v_{i^{\ast}_{j(\tau)}, j(\tau)}\mathbf{e}^{(i^{\ast}_{j(\tau)})}, \beta\rangle -\sum_{i\in N}B_i\log \beta_i\right\} \nonumber\\
    &= \argmin_{\beta \in \left[\frac{B}{h(1+\delta_0)}, \frac{1+\delta_0}{l}\mathbf{1}\right]} \left\{\frac{1}{t}\sum_{\tau=1}^t\langle g_{\tau}, \beta\rangle + \Psi(\beta)\right\},
    \label{eq:update_beta}
\end{align}
where $g_{\tau}=v_{i^{\ast}_{j(\tau)},j(\tau)}\mathbf{e}^{(i^{\ast}_{j(\tau)})}$.

From the first-order optimality condition for \eqref{eq:update_beta}, there exist a subgradient $b_{t+1}\in \partial \Psi(\beta^{t+1})$ such that
\begin{align}
    \label{eq:optimality_condition}
    \langle \sum_{\tau=1}^tg_{\tau} + tb_{t+1}, \beta^{\ast, \DA} - \beta^{t+1}\rangle \geq 0, ~\forall \beta\in \left[\frac{B}{h(1+\delta_0)}, \frac{1+\delta_0}{l}\mathbf{1}\right].
\end{align}
From Lemma \ref{lem:psi_strongly_convexity}, $\Psi$ is $\mu$-strongly convex with $\mu=\frac{l^2\min_{i\in N}B_i}{(1+\delta_0)^2}$, and then we have for any $\beta \in \left[\frac{B}{h(1+\delta_0)}, \frac{1+\delta_0}{l}\mathbf{1}\right]$,
\begin{align}
    \label{eq:strong_convexity}
    \Psi(\beta) \geq \Psi(\beta^{t+1}) + \langle b_{t+1}, \beta - \beta^{t+1}\rangle + \frac{\mu}{2}\|\beta^{t+1} - \beta\|^2.
\end{align}
By combining \eqref{eq:optimality_condition} and \eqref{eq:strong_convexity}, we get:
\begin{align*}
    &\frac{\mu t}{2}\|\beta^{t+1} - \beta^{\ast, \DA}\|^2 \\
    &\leq \langle \sum_{\tau=1}^tg_{\tau}, \beta^{\ast, \DA} - \beta^{t+1}\rangle + t\Psi(\beta^{\ast, \DA}) - t\Psi(\beta^{t+1}) \\
    &= \langle -\sum_{\tau=1}^tg_{\tau}, \beta^{t+1} - \beta^1\rangle - t\Psi(\beta^{t+1}) + t\Psi(\beta^{\ast, \DA}) + \langle \sum_{\tau=1}^tg_{\tau}, \beta^{\ast, \DA} - \beta^1\rangle \\
    &= \langle -\sum_{\tau=1}^tg_{\tau}, \beta^{t+1} - \beta^1\rangle - t\Psi(\beta^{t+1}) + t\Psi(\beta^{\ast, \DA}) + \sum_{\tau=1}^t\langle g_{\tau}, \beta^{\tau} - \beta^1\rangle + \sum_{\tau=1}^t\langle g_{\tau}, \beta^{\ast, \DA} - \beta^{\tau}\rangle \\
    &= \langle -\sum_{\tau=1}^tg_{\tau}, \beta^{t+1} - \beta^1\rangle - t\Psi(\beta^{t+1}) + \sum_{\tau=1}^t\left(\langle g_{\tau}, \beta^{\tau} - \beta^1\rangle + \Psi(\beta^{\tau})\right) + \sum_{\tau=1}^t\langle g_{\tau}, \beta^{\ast, \DA} - \beta^{\tau}\rangle + t\Psi(\beta^{\ast, \DA}) - \sum_{\tau=1}^t\Psi(\beta^{\tau}) \\
    &\leq \langle -\sum_{\tau=1}^tg_{\tau}, \beta^{t+1} - \beta^1\rangle - t\Psi(\beta^{t+1}) + \sum_{\tau=1}^t\left(\langle g_{\tau}, \beta^{\tau} - \beta^1\rangle + \Psi(\beta^{\tau})\right) \\
    &+ \sum_{\tau=1}^t\left(f(\beta^{\ast, \DA}, j(\tau)) - f(\beta^{\tau}, j(\tau))\right) + t\Psi(\beta^{\ast, \DA}) - \sum_{\tau=1}^t\Psi(\beta^{\tau}) \\
    &= \langle -\sum_{\tau=1}^tg_{\tau}, \beta^{t+1} - \beta^1\rangle - t\Psi(\beta^{t+1}) + \sum_{\tau=1}^t\left(\langle g_{\tau}, \beta^{\tau} - \beta^1\rangle + \Psi(\beta^{\tau})\right) \\
    &+ \sum_{\tau=1}^t\left(f(\beta^{\ast, \DA}, j(\tau)) + \Psi(\beta^{\ast, \DA})\right) - \sum_{\tau=1}^t\left(f(\beta^{\tau}, j(\tau)) + \Psi(\beta^{\tau})\right),
\end{align*}
where the second inequality follows from the convexity of $f(\cdot, j(\tau))$.
Here, let us define
\begin{align*}
    V_0(s) = \max_{\beta\in \left[\frac{B}{h(1+\delta_0)}, \frac{1+\delta_0}{l}\mathbf{1}\right]}\{\langle s, \beta - \beta^1\rangle - \Psi(\beta)\},
\end{align*}
and for $t\geq 1$,
\begin{align*}
    V_t(s) = \max_{\beta\in \left[\frac{B}{h(1+\delta_0)}, \frac{1+\delta_0}{l}\mathbf{1}\right]}\{\langle s, \beta - \beta^1\rangle - t\Psi(\beta)\}.
\end{align*}
Since $\beta^{t+1}=\argmax_{\beta\in \left[\frac{B}{h(1+\delta_0)}, \frac{1+\delta_0}{l}\mathbf{1}\right]}\{\langle -\sum_{\tau=1}^tg_{\tau}, \beta - \beta^1\rangle - t\Psi(\beta)\}$,
\begin{align}
    \label{eq:distance_ineq}
    &\frac{\mu t}{2}\|\beta^{t+1} - \beta^{\ast, \DA}\|^2 \nonumber\\
    &\leq V_t(-\sum_{\tau=1}^tg_{\tau}) + \sum_{\tau=1}^t\left(\langle g_{\tau}, \beta^{\tau} - \beta^1\rangle + \Psi(\beta^{\tau})\right) + \sum_{\tau=1}^t\left(f(\beta^{\ast, \DA}, j(\tau)) + \Psi(\beta^{\ast, \DA})\right) - \sum_{\tau=1}^t\left(f(\beta^{\tau}, j(\tau)) + \Psi(\beta^{\tau})\right).
\end{align}

For $\tau \geq 2$, we have:
\begin{align*}
    V_{\tau}(-\sum_{s=1}^{\tau}g_s) - V_{\tau - 1}(-\sum_{s=1}^{\tau}g_s) &\leq \langle -\sum_{s=1}^{\tau}g_s, \beta^{\tau+1} - \beta^1\rangle - \tau \Psi(\beta^{\tau+1}) - \langle -\sum_{s=1}^{\tau}g_s, \beta^{\tau+1} - \beta^1\rangle + (\tau-1) \Psi(\beta^{\tau+1}) \\
    &= -\Psi(\beta^{\tau+1}).
\end{align*}
Therefore,
\begin{align}
    \label{eq:v_sub}
    V_{\tau}(-\sum_{s=1}^{\tau}g_s) + \Psi(\beta^{\tau+1}) &\leq V_{\tau - 1}(-\sum_{s=1}^{\tau}g_s) = V_{\tau - 1}(-\sum_{s=1}^{\tau-1}g_s - g_{\tau}).
\end{align}
Here, since $t\Psi(\cdot)$ is $\mu t$-strongly convex, the convex conjugate of $t\Psi(\cdot)$ is $\frac{1}{\mu t}$-smooth.
Thus,
\begin{align}
    \label{eq:v_smooth}
    V_{\tau - 1}(-\sum_{s=1}^{\tau-1}g_s - g_{\tau}) \leq V_{\tau - 1}(-\sum_{s=1}^{\tau-1}g_s) - \langle \nabla V_{\tau - 1}(-\sum_{s=1}^{\tau-1}g_s), g_{\tau}\rangle + \frac{\|g_{\tau}\|^2}{2\mu(\tau-1)}.
\end{align}
On the other hand, the gradient of $V_{\tau}$ is given as:
\begin{align}
    \label{eq:nabla_v}
    \nabla V_{\tau-1}(-\sum_{s=1}^{\tau-1}g_s) &= \nabla \max_{\beta\in \left[\frac{B}{h(1+\delta_0)}, \frac{1+\delta_0}{l}\mathbf{1}\right]}\left\{\langle-\sum_{s=1}^{\tau-1}g_s, \beta \rangle - (\tau - 1)\Psi(\beta)\right\} - \beta^1 \nonumber\\
    &= \beta^{\tau} - \beta^1.
\end{align}
By combining \eqref{eq:v_sub}, \eqref{eq:v_smooth}, and \eqref{eq:nabla_v}, we have for $\tau\geq 2$:
\begin{align}
    \label{eq:g_tau}
    \langle \beta^{\tau} - \beta^1, g_{\tau}\rangle + \Psi(\beta^{\tau+1}) &\leq V_{\tau - 1}(-\sum_{s=1}^{\tau-1}g_s) - V_{\tau}(-\sum_{s=1}^{\tau}g_s) + \frac{\|g_{\tau}\|^2}{2\mu(\tau-1)}.
\end{align}
Similarly, for $\tau=1$, we have:
\begin{align*}
    V_1(-g_1) = V_0(-g_1) \leq V_0(0) - \langle \nabla V_0(0), g_1\rangle + \frac{\|g_1\|^2}{2\mu},
\end{align*}
and
\begin{align*}
    \nabla V_{0}(0) &= \nabla \max_{\beta\in \left[\frac{B}{h(1+\delta_0)}, \frac{1+\delta_0}{l}\mathbf{1}\right]}\left\{\langle 0, \beta \rangle - \Psi(\beta)\right\} - \beta^1 \\
    &= \nabla \max_{\beta\in \left[\frac{B}{h(1+\delta_0)}, \frac{1+\delta_0}{l}\mathbf{1}\right]}\left\{- \Psi(\beta)\right\} - \beta^1 \\
    &= \beta^1 - \beta^1 = 0.
\end{align*}
Thus,
\begin{align}
    \label{eq:g_1}
    \langle \beta^1 - \beta^1, g_1\rangle + \Psi(\beta^2) \leq V_0(0) - V_1(-g_1) + \frac{\|g_1\|^2}{2\mu} + \Psi(\beta^2).
\end{align}
By summing \eqref{eq:g_tau} and \eqref{eq:g_1} for $\tau=1,\cdots, t$,
\begin{align*}
    \sum_{\tau=1}^t\left(\langle \beta^{\tau} - \beta^1, g_{\tau}\rangle + \Psi(\beta^{\tau+1})\right) &\leq V_0(0) - V_t(-\sum_{\tau=1}^t g_{\tau}) + \Psi(\beta^2) + \frac{1}{2\mu}\left(\|g_1\|^2 + \sum_{\tau=2}^t \frac{\|g_{\tau}\|^2}{\tau - 1}\right) \\
    &= - V_t(-\sum_{\tau=1}^t g_{\tau}) + \Psi(\beta^2) - \Psi(\beta^1) + \frac{1}{2\mu}\left(\|g_1\|^2 + \sum_{\tau=2}^t \frac{\|g_{\tau}\|^2}{\tau - 1}\right).
\end{align*}
Since $\beta^1 = \argmin_{\beta\in \left[\frac{B}{h(1+\delta_0)}, \frac{1+\delta_0}{l}\mathbf{1}\right]}\Psi(\beta)$, $\Psi(\beta^{t+1})\geq \Psi(\beta^1)$ holds.
Thus, adding the inequality $\Psi(\beta^1) - \Psi(\beta^{t+1}) \leq 0$ to the above inequality, we get:
\begin{align}
    \label{eq:g_tau_sum}    
    \sum_{\tau=1}^t\left(\langle \beta^{\tau} - \beta^1, g_{\tau}\rangle + \Psi(\beta^{\tau})\right) &= \sum_{\tau=1}^t\left(\langle \beta^{\tau} - \beta^1, g_{\tau}\rangle + \Psi(\beta^{\tau+1})\right) + \Psi(\beta^1) - \Psi(\beta^{t+1}) \nonumber\\
    &\leq - V_t(-\sum_{\tau=1}^t g_{\tau}) + \Psi(\beta^2) - \Psi(\beta^1) + \frac{1}{2\mu}\left(\|g_1\|^2 + \sum_{\tau=2}^t \frac{\|g_{\tau}\|^2}{\tau - 1}\right).
\end{align}

Adding \eqref{eq:distance_ineq} and \eqref{eq:g_tau_sum} yields:
\begin{align}
    \label{eq:distance_ineq2} 
    &\frac{\mu t}{2}\|\beta^{t+1} - \beta^{\ast, \DA}\|^2 \nonumber\\
    &\leq \Psi(\beta^2) - \Psi(\beta^1) + \frac{1}{2\mu}\left(\|g_1\|^2 + \sum_{\tau=2}^t \frac{\|g_{\tau}\|^2}{\tau - 1}\right) + \sum_{\tau=1}^t\left(f(\beta^{\ast, \DA}, j(\tau)) + \Psi(\beta^{\ast, \DA})\right) - \sum_{\tau=1}^t\left(f(\beta^{\tau}, j(\tau)) + \Psi(\beta^{\tau})\right).
\end{align}

Since $\beta^2 = \argmin_{\beta\in \left[\frac{B}{h(1+\delta_0)}, \frac{1+\delta_0}{l}\mathbf{1}\right]}\left\{\langle g_1, \beta\rangle + \Psi(\beta)\right\}$, we have:
\begin{align*}
    \langle g_1, \beta^2\rangle + \Psi(\beta^2) \leq \langle g_1, \beta^1\rangle + \Psi(\beta^1).
\end{align*}
Thus, from the Cauchy–Schwarz inequality, 
\begin{align*}
    \Psi(\beta^2) - \Psi(\beta^1) \leq \langle g_1, \beta^1 - \beta^2\rangle \leq \|g_1\|\|\beta^1 - \beta^2\|.
\end{align*}
On the other hand, by strong convexity of $\Psi$ and the first-order optimality condition for $\beta^1$, we have:
\begin{align*}
    \Psi(\beta^2) \geq \Psi(\beta^1) + \frac{\mu}{2}\|\beta^1 - \beta^2\|^2.
\end{align*}
By combining the above inequalities, we have:
\begin{align*}
    \frac{\mu}{2}\|\beta^1 - \beta^2\|^2 \leq \|g_1\|\|\beta^1 - \beta^2\|,
\end{align*}
and then:
\begin{align*}
    \|\beta^1 - \beta^2\| \leq \frac{2}{\mu}\|g_1\|.
\end{align*}
Therefore, we have:
\begin{align}
    \label{eq:psi_diff}
    \Psi(\beta^2) - \Psi(\beta^1) \leq \frac{2\|g_1\|^2}{\mu}.
\end{align}
By combining \eqref{eq:distance_ineq2} and \eqref{eq:psi_diff}, we get:
\begin{align}
    \label{eq:distance_ineq3}
    \|\beta^{t+1} - \beta^{*,\DA}\|^2 &\leq \frac{2}{\mu t}\left(\frac{2\|g_1\|^2}{\mu} + \frac{1}{2\mu}\left(\|g_1\|^2 + \sum_{\tau=2}^t \frac{\|g_{\tau}\|^2}{\tau - 1}\right)\right) \nonumber\\
    &+ \frac{2}{\mu t}\left(\sum_{\tau=1}^t\left(f(\beta^{*,\DA}, j(\tau)) + \Psi(\beta^{*,\DA})\right) - \sum_{\tau=1}^t\left(f(\beta^{\tau}, j(\tau)) + \Psi(\beta^{\tau})\right)\right).
\end{align}

Since the random variables $(j(1), \cdots, j(t))$ is i.i.d, the expectation of $f(\beta^{\tau}, j(\tau)) + \Psi(\beta^{\tau})$ is given as:
\begin{align*}
    \mathbb{E}_{(j(s))_{s=1}^t}\left[f(\beta^{\tau}, j(\tau)) + \Psi(\beta^{\tau})\right] &= \mathbb{E}_{(j(s))_{s=1}^{\tau-1}}\left[\mathbb{E}_{(j(s))_{s={\tau}}^t}\left[f(\beta^{\tau}, j(\tau)) + \Psi(\beta^{\tau}) ~|~ (j(s))_{s=1}^{\tau-1}\right]\right] \\
    &= \mathbb{E}_{(j(s))_{s=1}^{\tau-1}}\left[\mathbb{E}_{j(\tau)}\left[f(\beta^{\tau}, j(\tau)) \right] + \Psi(\beta^{\tau})\right] \\
    &= \mathbb{E}_{(j(s))_{s=1}^{\tau-1}}[\phi(\beta^{\tau})].
\end{align*}
Similarly, the expectation of $f(\beta^{*,\DA}, j(\tau)) + \Psi(\beta^{*,\DA})$ is given by:
\begin{align*}
    \mathbb{E}_{(j(s))_{s=1}^t}\left[f(\beta^{*,\DA}, j(\tau)) + \Psi(\beta^{*,\DA})\right] = \mathbb{E}_{j(\tau)}\left[f(\beta^{*,\DA}, j(\tau)) + \Psi(\beta^{*,\DA})\right] = \phi(\beta^{*,\DA}).
\end{align*}
Thus, the expected regret is upper bounded as:
\begin{align}
    \label{eq:expected_regret_ub}
    \mathbb{E}_{(j(s))_{s=1}^t}\left[\sum_{\tau=1}^t \left(f(\beta^{*,\DA}, j(\tau)) + \Psi(\beta^{*,\DA})\right) - \sum_{\tau=1}^t f(\beta^{\tau}, j(\tau)) + \Psi(\beta^{\tau})\right] &= \sum_{\tau=1}^t\left(\phi(\beta^{*,\DA}) - \mathbb{E}_{(j(s))_{s=1}^{\tau-1}}[\phi(\beta^{\tau})]\right) \leq 0.
\end{align}
From \eqref{eq:expected_regret_ub}, the expectation of \eqref{eq:distance_ineq3} is given by:
\begin{align*}
    &\mathbb{E}_{(j(s))_{s=1}^t}\left[\|\beta^{t+1} - \beta^{*,\DA}\|^2\right] \\
    &\leq \frac{2}{\mu t}\mathbb{E}_{(j(s))_{s=1}^t}\left[\frac{2\|g_1\|^2}{\mu} + \frac{1}{2\mu}\left(\|g_1\|^2 + \sum_{\tau=2}^t \frac{\|g_{\tau}\|^2}{\tau - 1}\right)\right] \\
    &= \frac{2}{\mu t}\left(\frac{2\mathbb{E}_{j(1)}[\|v_{i^{\ast}_{j(1)},j(1)}^{\DA}\mathbf{e}^{(i^{\ast}_{j(1)})}\|^2]}{\mu} + \frac{1}{2\mu}\left(\mathbb{E}_{j(1)}[\|v_{i^{\ast}_{j(1)},j(1)}^{\DA}\mathbf{e}^{(i^{\ast}_{j(1)})}\|^2] + \sum_{\tau=2}^t \frac{\mathbb{E}_{(j(s))_{s=1}^{\tau}}[\|v_{i^{\ast}_{j(\tau)},j(\tau)}^{\DA}\mathbf{e}^{(i^{\ast}_{j(\tau)})}\|^2]}{\tau - 1}\right)\right) \\
    &= \frac{2}{\mu t}\left(\frac{2\max_{i\in N}\mathbb{E}_{j\sim s}[(v_{i,j}^{\DA})^2]}{\mu} + \frac{1}{2\mu}\left(\max_{i\in N}\mathbb{E}_{j\sim s}[(v_{i,j}^{\DA})^2] + \sum_{\tau=2}^t \frac{\max_{i\in N}\mathbb{E}_{j\sim s}[(v_{i,j}^{\DA})^2]}{\tau - 1}\right)\right) \\
    &\leq \frac{G^2}{\mu^2 t}\left(5 + \sum_{\tau=2}^t\frac{1}{\tau-1}\right) = \frac{G^2}{\mu^2 t}\left(5 + \sum_{\tau=1}^{t-1}\frac{1}{\tau}\right) \leq \frac{G^2}{\mu^2 t}(6 + \log t).
\end{align*}
\end{proof}

\section{More details on the experiments}
In this section, we provide the result of Household ($n$=50) in Figure \ref{fig:regret-n50} and the l2-loss of utilities $||\bar{u}_i - u_i^*||$ in Table \ref{tab:u_loss}.

Figure~\ref{fig:regret-n50} depicts regrets with $n=50$ for the Household dataset. Inevitably, the same tendency retains as in Figures~\ref{fig:all-wide} and \ref{fig:all-narrow}. 
Table~\ref{tab:u_loss} illustrates the l2-loss of utilities $||\bar{u}_i - u_i^*||$. This metric measures the disparity between the true and the resulting utilities. DA-EtC or DA-UCB outperformed the others. The inconsistency of DA-Grdy was more apparent than the regret indicated in the figures. This observation implies that the estimated values in DA-Grdy are more inaccurate than those of DA-EtC and DA-UCB.

\begin{figure}[H]
\begin{minipage}[t]{0.48\linewidth}
    \centering
    \includegraphics[keepaspectratio,width=\linewidth]{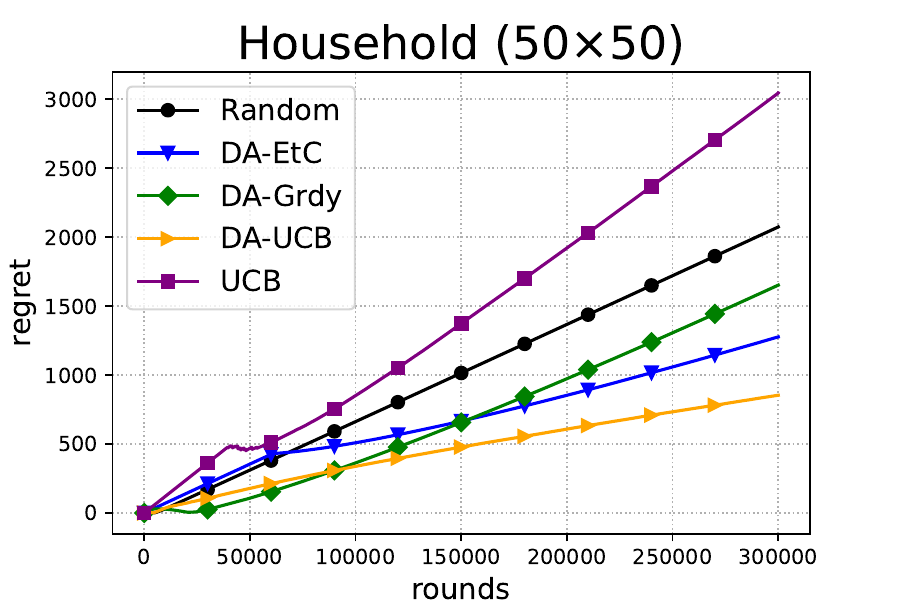}
    \subcaption{All algorithms}\label{test1}
\end{minipage}
\begin{minipage}[t]{0.48\linewidth}
    \centering
    \includegraphics[keepaspectratio,width=\linewidth]{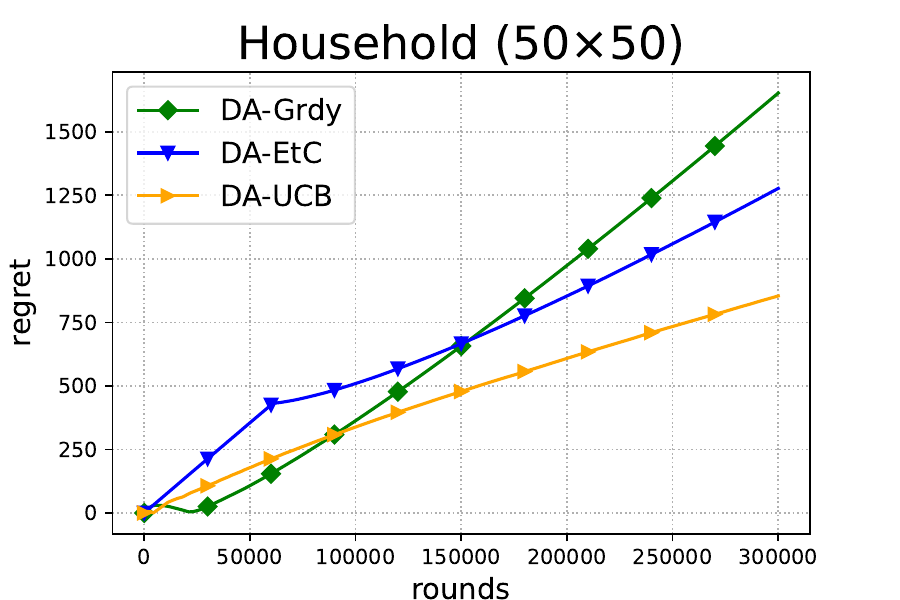}
    \subcaption{Three best algorithms}\label{test2}
\end{minipage}
\caption{Regret in Household with a larger number of agents ($n=50$).
}\label{fig:regret-n50}
\end{figure}

\begin{table}[H]
\centering
\caption{Loss of utilities per round. Type boldface indicates the minimum loss at the dataset.}\label{tab:u_loss}
\begin{tabular}{l|ccccc} 
 & Uniform & Jester & Household & Household 
\\
 &  &  & ($n=10$) & ($n=50$)
\\ \hline
Random & 0.041 & 0.031 & 0.033 & 0.006\\
UCB & 0.073 & 0.073 & 0.070 & 0.019\\
DA-Grdy & 0.010 & 0.008 & 0.015 & 0.006\\
DA-EtC & 0.004 & \textbf{0.007} & 0.005 & 0.004\\
DA-UCB & \textbf{0.002} & 0.008 & \textbf{0.004} & \textbf{0.003}\\
\end{tabular}
\end{table}

\end{CJK}
\end{document}